\documentclass[twoside]{article}

\usepackage[preprint]{aistats2026}
\makeatletter
\renewcommand{\Notice@String}{%
\footnotemark[1]Corresponding author: \texttt{jiayul11@illinois.edu}. \\
\footnotemark[2]Work done while working at Betterdata AI and National University of Singapore. \\
Proceedings of the 29\textsuperscript{th} International Conference on Artificial Intelligence and Statistics (AISTATS) 2026, Tangier, Morocco.
PMLR: Volume 300. Copyright 2026 by the author(s).}
\makeatother

\usepackage{times}  
\usepackage{helvet}  
\usepackage{courier}  
\usepackage[hyphens]{url}  
\usepackage{graphicx} 
\usepackage{natbib}  
\usepackage{caption}
\usepackage{algorithm}
\usepackage{algorithmic}
\usepackage{graphicx}
\usepackage{subfigure}
\usepackage{booktabs} 
\usepackage{caption}
\usepackage{amsmath}
\usepackage{amssymb}
\usepackage{amsthm}
\usepackage{soul}     
\usepackage{comment}
\usepackage{multirow}
\usepackage{enumitem}
\usepackage{xspace}
\usepackage{tablefootnote}
\usepackage{xcolor}
\usepackage{colortbl}
\usepackage{cuted}  
\usepackage{array}

\usepackage{newfloat}
\usepackage{listings}

\newtheorem{lemma}{Lemma}
\newtheorem{theorem}{Theorem}
\newtheorem{remark}{Remark}
\newtheorem{proposition}{Proposition}
\newcommand{\frameworkname}{TabTreeFormer}
\newcommand{\framework}{TabTreeFormer\xspace}
\newcommand{\edit}{}

\begin{document}

%

%
\runningauthor{Li, Zhao, Zhao, Javaid, Yee, Sikdar}

\renewcommand{\thefootnote}{\fnsymbol{footnote}}

\twocolumn[

\aistatstitle{{TabTreeFormer}: Tabular Data Generation Using Hybrid Tree-Transformer}

\aistatsauthor{
Jiayu Li\footnotemark[1]\footnotemark[2] \\
{\normalfont \small University of Illinois, Urbana-Champaign}
\And
Bingyin Zhao\footnotemark[2] \\
{\normalfont \small Pixocial Technology}
\And
Zilong Zhao \\
{\normalfont \small Betterdata AI}
\AND
Uzair Javaid \\
{\normalfont \small Betterdata AI}
\And
Kevin Yee \\
{\normalfont \small Betterdata AI}
\And
Biplab Sikdar \\
{\normalfont \small National University of Singapore}
}
\vspace{0.15in}
]




\begin{abstract}
Transformers have shown impressive results in tabular data generation. However, they lack domain-specific inductive biases
which are critical for preserving the intrinsic characteristics of tabular data. They also suffer from poor scalability and efficiency due to quadratic computational complexity. In this paper, we propose \frameworkname, a hybrid transformer architecture that integrates inductive biases of tree-based models (e.g., non-smoothness and non-rotational invariance) to effectively handle the discrete and weakly correlated features in tabular datasets. To improve numerical fidelity and capture multimodal distributions, we introduce a novel tokenizer that learns token sequences based on the complexity of tabular values. This reduces vocabulary size and sequence length, yielding more compact and efficient representations without sacrificing performance. We evaluate \framework on nine diverse datasets, benchmarking against eight generative models. We show that \framework consistently outperforms baselines in utility, fidelity, and privacy metrics with competitive efficiency. Notably, in scenarios prioritizing data utility over privacy and efficiency, the best variant of \framework delivers a 44\% performance gain relative to its baseline variant. Our code is available at: \url{https://github.com/li-jiayu-ljy/tabtreeformer}.

\end{abstract}
\section{Introduction}
\label{sec:intro}
Tabular data is a prevalent data modality in real-world applications (e.g., healthcare~\citep{healthcare}, financial services~\citep{financial}, etc.), yet are heavily under-exploited due to privacy concerns~\citep{gdpr}. Fortunately, synthetic data offers an alternative to the utilization of tabular data by modeling the characteristics of real data and reducing the risk of data breach~\citep{pdpc}, thus drawing considerable attention in recent years. 

State-of-the-art (SOTA) research shows that transformers such as autoregressive transformers~\citep{great}, masked transformers~\citep{tabmt}, and diffusion models with transformers~\citep{tabsyn} have achieved impressive performance in tabular data generation, allowing synthetic data to empower a variety of fields~\citep{syn-healthcare,financial}. However, unlike the application of transformers in other research areas such as computer vision~\citep{vit} and natural language processing~\citep{transformer}, existing transformer models for synthetic tabular data often overlook domain-specific priors (i.e., inductive biases). For example, CvT~\citep{cvt} introduces convolutional embedding and projection to transformers to boost vision performances, and Transformer-XL~\citep{transformer-xl} introduces segment-level recurrence to transformers to capture longer-term dependencies. While vision and language models have enjoyed the performance boost introduced by domain-specific priors, less exploration
has been conducted to leverage them in tabular generative transformers. Moreover, transformers suffer from poor scalability due to quadratic computational complexity. This raises two interesting questions:
\begin{enumerate}
    \item What inductive biases are beneficial to the quality of synthetically generated tabular data?
    \item How can one exploit these inductive biases to improve the generative transformer?
\end{enumerate}

\begin{figure}
    \centering
    \includegraphics[width=\linewidth,trim={8em 49em 8em 6em},clip]{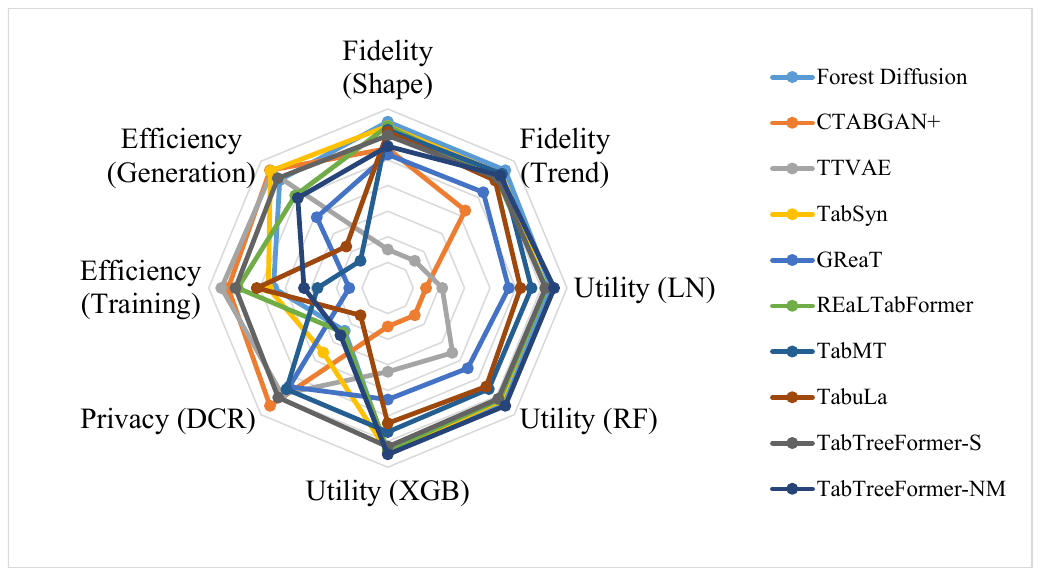}
    \caption{Performance comparison of \framework (Ours) with SOTA tabular generative models in utility, fidelity, privacy, and efficiency metrics. \framework-S achieves the best balance as the \ul{only big near-regular octagon}, and \framework-NM achieves the best utility.
    }
    \label{fig:radar}
\end{figure}

To answer the questions, we propose \frameworkname, a hybrid transformer incorporating a tree-based model and a novel tokenizer to handle tabular-specific inductive biases. We also leverage the limited per-dimension semantic meaning (i.e., each dimension corresponds to at most one feature) of typical tabular data to enable a significant reduction in model size with our tokenizer design. 

Tree-based models excel at tabular classification and regression tasks~\citep{tree-table,revisit}, which is attributed to their inductive biases that can capture tabular characteristics such as \textit{non-smoothness} and potentially \textit{low-correlation}~\citep{tab-vs-nn}. Non-smoothness is due to the existence of discrete features and non-smooth relations between discrete and/or continuous features. Trees effectively model these by learning piecewise constant functions, unlike neural networks that typically learn smoother, low-frequency functions~\citep{low-freq}. Low-correlated features, often uninformative, contribute minimally to downstream tasks such as classification and regression. Due to the non-rotationally invariant nature of tree-based models~\citep{rotinv}, trees are more robust against low-correlated features, while neural networks are biased towards stronger correlations everywhere. Inspired by the success of tree models on tabular tasks, we propose to employ such priors to facilitate the performance on tabular generative tasks.

However, tree models do not inherently capture \textit{multimodal distributions} in continuous features (i.e., features whose probability density functions have multiple modes or peaks). This limitation poses a significant challenge for tabular generative modeling. Inspired by~\citep{ctgan} that addresses this issue via multimodal decomposition, we propose dual-quantization tokenization for transformers. The first quantization uses K-Means clustering~\citep{kmeans} to model multimodal distributions, while the second employs a separate quantile-based quantization to achieve more precise representation of numerical values. Unlike typical word tokens, quantized tokens possess ordinal relationships. To effectively capture these relationships, we design and implement custom embeddings and loss functions.

Our contributions are summarized as follows: 
\begin{itemize} [topsep=0pt, partopsep=0pt, itemsep=0pt, parsep=0pt,leftmargin=*]
    \item To the best of our knowledge, we are the first to introduce tabular-specific inductive biases to transformers with a tree-based model for improving table generation. 
    \item We propose a dual-quantization tokenizer with ordinal-aware embeddings and loss functions, enabling transformers to model multimodal continuous features and generate high-quality synthetic data.
    \item Compared to 8 baselines across 9 datasets, \framework achieves an outstanding balance in quality, privacy, and efficiency, as shown in Fig.~\ref{fig:radar}. It also achieves up to 54\% and 44\% gain over the best deep and general baselines respectively, when prioritizing utility over other metrics.
\end{itemize}

\section{Related Works}
\label{sec:background}
\subsection{Tabular Data Generation}
\label{sec:background:tdg}
Early works on tabular data generation use MLPs and CNNs as backbone architectures with GAN and VAE as generation methods~\citep{tablegan,ctgan,ctabgan,ctabganp}. Recent works now use transformers with auto-regression, masked modeling, and diffusion as generative paradigms~\citep{great,tabsyn,tabmt}. 
For example, TabMT~\citep{tabmt} employs a masked transformer with ordered embedding and achieves good utility and scalability in downstream tasks. TabSyn~\citep{tabsyn} encodes tabular data into a latent space and generates tabular data using a diffusion model, demonstrating impressive fidelity and utility. Despite their enhanced performance, they overlook inductive biases to capture non-smoothness and low-correlated features that are crucial to preserving the intrinsic characteristics of tabular data. Moreover, \citet{ctgan} highlight that capturing {multimodal distribution} in continuous features is a critical challenge for tabular generative models. CTGAN~\citep{ctgan} uses variational Gaussian mixture model~\citep{vgm} to decompose multimodal values, while TabDiff~\citep{tabdiff} introduces a multimodal stochastic sampler. In this paper, we propose a tokenizer that models the multimodal distribution and injects this prior into an auto-regressive transformer for improved tabular data generation.

\subsection{Tree-based Models for Tabular Data}
\label{sec:background:tbm}

Tree-based models like XGBoost~\citep{xgboost}, LightGBM~\citep{lgbm}, and CatBoost~\citep{catboost} dominate tabular tasks due to their strong inductive biases, enabling high predictive performance, efficiency, and scalability.
While transformer-based models~\citep{tabtransformer,tabnet,tabpfn} show promise, they often underperform due to lack of such biases~\citep{tree-table,revisit}. Tree-based generative models~\citep{arf,utrees,fdiff}, though efficient, struggle with synthetic data quality or suffer from privacy leakage risks. In this paper, we bridge the gap by combining tree-based priors with transformers to inject tabular-specific inductive bias and enhance tabular generative performance.

\subsection{Tabular Tokenizers}
Transformers require a tokenizer to convert tabular data into suitable input tokens. Some methods map data to a continuous space by redesigning embedding layers~\citep{tabnet,ple,tabsyn}, while others tokenize into discrete sequences with minimal embedding changes~\citep{great,realtabformer,tabmt,tabula}, which align better with NLP practices. In this paper, we adopt the latter to optimize and simplify tabular tokenization.
This approach converts continuous values to discrete tokens, which have 
ordinal relations, such that the greatness of the token IDs can be compared. 
To capture this relation, two aspects where a model can be modified to cater for the ordinal token space are: i) embedding (input, also output for causal language models) and ii) loss (output).

\paragraph{Ordinal Embedding.}
One of the most well-known continuous embedding methods is the positional embedding based on trigonometric functions~\citep{transformer} that has been adopted for ordinal tokens~\citep{cat2vec}. However, their periodic nature is better suited for language positions than general ordinal data. Other methods often rely on non-ordinal embeddings~\citep{ordinal-constraint,tabmt} or require high-precision raw values~\citep{ple,tabmt}. We propose a function-based ordinal embedding that sidesteps both demands.

\paragraph{Ordinal Loss.}
Prior work on ordinal regression often targets a small number of classes~\citep{ord-reg,ord-emd,coral,corn}, or modifies cross-entropy loss to account for token distances~\citep{sol,slace,coling,ord-entropy}. Building on core ideas from prior work, we propose a custom loss tailored to our setting. 

\begin{figure*}
\centering
\centering    
\includegraphics[width=1\linewidth]{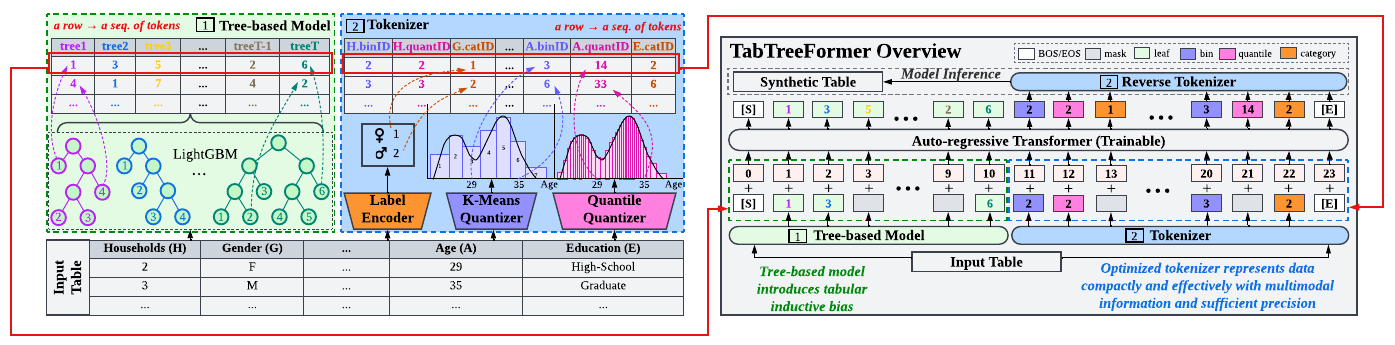} 
\caption{Overview of \framework (data flow: left $\rightarrow$ right, bottom $\rightarrow$ top). It consists of 3 components: i) \textbf{a tree-based model} that introduces tabular-specific inductive biases; ii) \textbf{a tokenizer} that efficiently and compactly represents data while capturing multimodal distributions; iii) \textbf{a transformer model} that learns the priors extracted from the tree and tokenizer to generate  synthetic data. 
} \label{fig:overview} 
\end{figure*}
\section{\frameworkname}
\label{sec:algo}

Our goal is to improve generative modeling of tabular data by introducing domain-specific inductive biases. Let $\mathbf{X}$ be the training dataset with $n$ rows, $m_d$ discrete, and $m_c$ continuous features. To generate synthetic tabular data $\mathbf{X}'$ similar to $\mathbf{X}$, we propose \framework, a model composed of three components: a tree-based model, a tokenizer, and a transformer, as shown in Fig.~\ref{fig:overview}. The tree model encodes tabular-specific inductive biases; the tokenizer captures multimodal distributions while reducing vocabulary and sequence length; and the transformer learns priors from both to generate high-quality synthetic data. More details on training and generation are provided in Appendix.

\subsection{Tree-based Model}
\label{sec:algo:tbm}

To inject tabular-specific inductive biases, we incorporate a tree-based model into the transformer (``Tree-based Model'' in Fig.~\ref{fig:overview}) by augmenting each row's token sequence with leaf indices from a fitted tree model $\mathcal{T}$ with $T$ trees. 
Let $l_k\in\mathbb{N}$ be the number of leaves in $k$-th tree, where $k\in\{1,\dots,T\}$.  In tree $\mathcal{T}$, the leaf index matrix $\mathbf{J} = [j_{ik}] \in \mathbb{N}^{n \times T}$ captures the position of each row $\mathbf{X}_i$ in each tree. These indices encode non-smooth and non-rotationally invariant structure, enabling the model to possess these inductive biases of tabular data. By prepending leaf indices to input tokens, we transfer tree-based inductive biases to the transformer. During inference, these indices also act as prompts and guide towards more realistic data generation.

\edit{
Many early tabular generation frameworks use conditional generation~\citep{ctgan,ctabgan}, which improves the generation quality noticeably.
These conditions are based on the values of a specific column in the dataset.
In TabTreeFormer, we extend the condition of generation to clusters of the data and allow multiple concurrent conditions in place.
Conceptually, each tree of a well-trained tree-based model provides an informative clustering of the training set, and the leaf index matrix provides a number of such clusterings.
Therefore, providing the leaf indices as prompts, serving as conditions for generation, where multiple concurrent conditions (clusters) are provided and each condition features the clustering on one tree, is expected to be an effective generation performance booster.
}

\begin{table*}[t]
    \centering
    \footnotesize
    \caption{Illustration of tokens in \frameworkname. ``Size'' indicates the number of tokens of this type. Recall Fig.~\ref{fig:overview} for the their colors and usage. ``N'' in ``Format'' can be any index in the range; e.g., leaf 3 ($3<n_l$) is represented as [leaf3], corresponding to the light green tokens in Fig.~\ref{fig:overview} with value 3.}
    \label{tab:tokens}
    \setlength{\tabcolsep}{1pt}
    \resizebox{\linewidth}{!}{
    \begin{tabular}{lccp{0.35\linewidth}p{0.28\linewidth}}
        \toprule
        Type (Alias) & Size & Format & Source & Description \\
        \midrule
        Leaf & $n_l=\max_{i\in\{1,2,\dots,T\}}l_i$ & \begin{minipage}{0.07\linewidth}
                \centering
                \includegraphics[width=\linewidth,trim={107pt 103pt 114pt 15pt},clip]{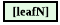}
             \end{minipage} & Tree-based model $\mathcal{T}$ & Leaf index in a tree\\
        Categorical (cat) & $n_c=\max_{i\in\{1,2,\dots,n\}}c_i$ & \begin{minipage}{0.07\linewidth}
                \centering
                \includegraphics[width=\linewidth,trim={208pt 103pt 13pt 15pt},clip]{Fig/tokens.pdf}
             \end{minipage} & Data tokenizer $\mathcal{D}$ (categorical columns) & Category ID \\
        Cluster (bin) & $n_b=\max_{i\in\{1,2,\dots,n\}}b_i$ & \begin{minipage}{0.07\linewidth}
                \centering
                \includegraphics[width=\linewidth,trim={142pt 103pt 79pt 15pt},clip]{Fig/tokens.pdf}
             \end{minipage} & Data tokenizer $\mathcal{D}$ (numeric columns) & Bin ID in K-Means quantizer \\
        Quantile (quant) & $n_q=\max_{i\in\{1,2,\dots,n\}}q_i$ & \begin{minipage}{0.07\linewidth}
                \centering
                \includegraphics[width=\linewidth,trim={175pt 103pt 46pt 15pt},clip]{Fig/tokens.pdf}
             \end{minipage} & Data tokenizer $\mathcal{D}$ (numeric columns) & Quant. ID in quant. quantizer \\
        Special & 3 & \begin{minipage}{0.21\linewidth}
                \centering
                \includegraphics[width=\linewidth,trim={22pt 103pt 142pt 15pt},clip]{Fig/tokens.pdf}
             \end{minipage} & Classical special tokens for LMs & BOS, EOS, and mask tokens\\
        \bottomrule
    \end{tabular}
    }
\end{table*}

\subsection{Tokenizer}
\label{sec:algo:tok}
To model multimodal distributions and reduce both vocabulary size and sequence length, we design a tailored dual-quantization tokenizer for tabular data (``Tokenizer'' in Fig.~\ref{fig:overview}). For each continuous feature, we quantize the values so that it is of a more natural data format for transformers to learn. However, unlike TabMT~\citep{tabmt} that uses a single quantizer, we use a dual-quantizer (i.e., double quantizers) to capture multimodal continuous distribution and describe the precise values respectively. Transformers do not expect near-Gaussian-distributed input tokens, so we \ul{capture the multimodal distribution simply by a K-Means quantizer}~\citep{kmeans}\ul{ with a small number of clusters (e.g., $K=10$)}. Following that, \ul{the original value is then quantized into a large number of quantiles (e.g., $Q=1000$), based on which original values are recovered, and thus the values are highly precise}.
Employing the dual-quantization tokenization scheme, each numeric value is encoded into two discrete values: a bin ID in the K-Means quantizer and a quantile ID in the quantile quantizer, allowing two tokens to represent a numeric value with valuable and sufficient information. Categorical features are simply tokenized by label encoding, requiring 1 token.
Formally, a reversible data tokenizer $\mathcal{D}$ can be fitted on $\mathbf{X}$. Let $c_i\in\mathbb{N},b_i\in\{0,\dots,K\},q_i\in\{0,\dots,Q\}$ denote the number of categories, K-Means bins, and quantiles in $i$-th column respectively, where $c_i=0$ in continuous features and $b_i=q_i=0$ in categorical features. Then, the domain for a categorical value is $\{1,\dots,c_i\}\subseteq\mathbb{N}^+$, and the domain for a continuous value is $\{1,\dots,b_i\}\times\{1,\dots,q_i\}\subseteq(\mathbb{N}^+)^2$.
A summary is seen in Table~\ref{tab:tokens}.

\subsection{Transformer}
\label{sec:model:lm}

\framework is simply trained as an auto-regressive transformer. 
The input token sequence is constructed by concatenating the outcome from the tree model and tokenizer. It requires 5 types of tokens and a vocabulary size of $V=n_l+n_c+n_b+n_q+3$. The number of tokens in a sequence corresponding to one row is $L=2+T+m_d+2m_c$, including BOS and EOS tokens.
\ul{Both the vocabulary size and sequence length are much smaller than the need for natural language models}. A shared set of leaf tokens across all trees and similarly shared sets of category, bin, and quantile tokens across all features are used. Distinctions between different trees and features are effectively encoded through positional information, leveraging the positional embeddings inherent in most transformers.

A key observation on the construction of token sequence is that the valid range of tokens at each position is fixed for a given table. Consequently, tokens can be sampled exclusively from this precomputed set of valid options during generation. This ensures that every generated token sequence corresponds to a valid row, \ul{eliminating the need for rejecting invalid sequences}~\citep{great,realtabformer,tabula}\ul{, thus significantly accelerating the sampling process}.

Additionally, transformers are prone to memorization, particularly when trained on tabular data, where datasets are much smaller compared to typical natural language corpora~\citep{great}. \ul{To avoid memorization, we heavily mask} the input and maintain the target output unmasked. 
Moreover, we evenly split the training set into 2 subsets, one neural network (i.e., transformer in TabTreeFormer) is trained on one subset and validated on the other.
Consequently, \ul{the validation loss can be a criterion for early stopping to avoid overfitting}, too.
When generating samples, we sample from the two networks with equal probability.
\begin{theorem}
\label{thm:split}
    Given a real dataset $\mathbf{X}$, train $M$ generative models ($\mathcal{G}_1,\mathcal{G}_2,\dots,\mathcal{G}_M$) on an $M$-partition of it ($\mathbf{X}=[\mathbf{X}^{[1]};\mathbf{X}^{[2]};\dots;\mathbf{X}^{[M]}]$). If the generators are well-trained (distribution $p$ of the corresponding partition is learned), then to sample $\mathbf{X}'$, sample data from $\mathcal{G}_i$ for a probability of $\left|\mathbf{X}^{[i]}\right|/|\mathbf{X}|$ where $|\cdot|$ means the number of rows, we will have the resulting $p(\mathbf{X}')=p(\mathbf{X})$.
\end{theorem}
Theorem~\ref{thm:split} validates a generalized case of the method with two evenly split subsets described above, and can be trivially proven by the chain rule, so the proof is omitted.
The rest of the training and generation settings are similar to the training and generation of classical causal language models, except for a modified embedding and loss function to account for the ordinal relations between quantile tokens, with details introduced in next Section.

\subsection{Embedding and Loss Function for Ordinal Tokens}
\label{sec:algo:loss}

\paragraph{Quantile Embeddings (QE) with Ordinal Relations.}
The known ordinal relation between quantile tokens resembles the known relative sequential relation between positions. Therefore, analogous to function-generated positional embeddings~\citep{transformer}, we use function-generated embeddings for quantile tokens. 
Unlike positional tokens, which emphasize relative values, quantile tokens focus on absolute values, so we \ul{replace periodic trigonometric functions with non-linear monotonic scaled sigmoid functions for embedding generation}. 
The embedding value generators are provided in Equation~\ref{eq:bin-emb}, where $i\in\{0,\dots,Q-1\}$ stands for the quantile ID, and $d\in\{0,\dots,D-1\}$ stands for the embedding dimensions.
More detailed intuition and description of the function are provided in Appendix.

\begin{align}
  S_d &= \bigg\lfloor \frac{1+\sqrt{1+4d}}{2} \bigg\rfloor \in \mathbb{N}^+
  & \text{(scale factor)} \label{eq:bin-scale}
\end{align}
\vspace{-1.2em}
\begin{align}
  O_d &= \frac{-4S_d^3+(4d+2)S_d}{2S_d-1} \in [-2S_d,2S_d]
  & \text{(offset)} \label{eq:bin-offset}
\end{align}
\vspace{-1.2em}
\begin{align}
     QE_{id}&=\mathrm{sigmoid}\left(4S_d\left(\frac{i}{Q}-\frac{1}{2}\right)+O_d\right)
    \label{eq:bin-emb}
\end{align}

\begin{theorem}
    Following the quantile embedding values in Equation~\ref{eq:bin-emb}, and let the embedded vector of quantile $i$ be $\mathbf{q}_i=\begin{pmatrix}QE_{i0}&\cdots&QE_{i(D-1)}\end{pmatrix}\in(0,1)^D$.
    For any $p\ge1$, given quantile IDs $i,j,k\in\{0,\dots,Q-1\}$, and $j,k$ are on the same side of $i$ (i.e., $(i-j)(i-k)\ge0$), then $|i-j|<|i-k|$ if and only if $\|\mathbf{q}_i-\mathbf{q}_j\|_p<\|\mathbf{q}_i-\mathbf{q}_k\|_p$, where $\|\cdot\|_p$ stands for the $p$-norm of a vector. 
    \label{thm:emb}
\end{theorem}

Theorem~\ref {thm:emb} shows that the trend of the distance between the embedded vectors is consistent with the difference between quantile IDs. 
The mathematical proof is provided in Appendix.
At different embedding dimensions, we apply different slopes (by scale factor) and intercepts (by offset) on the input to $\mathrm{sigmoid}$.
Note that while the meaning of positions in a transformer is fixed, the interpretation of quantiles can vary across datasets. As a result, instead of fixing the embedding values as in Equation~\ref{eq:bin-emb}, we initialize the values by the equations, and they are updated during training.
Non-quantile tokens will use a typical word token embedding layer of transformers.

\begin{table*}[t]
    \centering
    \caption{
    Averaged MLE performance of different models on all 9 datasets. RE stands for relative error (w.r.t. score on real data). Avg. stands for the average raw MLE score. 
    The best scores and the second best scores are highlighted in bold with and without underscore, respectively. 
    }
    \label{tab:ml-utility-summary}
    \setlength{\tabcolsep}{1pt}
    \resizebox{\linewidth}{!}{{
    \renewcommand{\arraystretch}{1.1}
    
\begin{tabular}{ll>{\columncolor{gray!20}}c c c c c c c c c >{\columncolor{red!10}}c >{\columncolor{red!10}}c >{\columncolor{red!10}}c}
\toprule
 & ML & Real & FD & CTAB+ & TTVAE & TabSyn & GReaT & RTF & TabMT & TabuLa\tablefootnote{For TabuLa, 2/9 datasets fail to generate reasonable data under the default setting.} & TTF-S & TTF-L & TTF-NM \\
\midrule
\multirow{4}{*}{RE\tablefootnote{The reported utility improvement in abstract and introduction is computed from this row.} ($\downarrow$)} & all & \multirow{4}{*}{-} & $\boldsymbol{0.020}_{\pm 0.033}$ & $0.253_{\pm 0.287}$ & $0.203_{\pm 0.335}$ & $0.024_{\pm 0.030}$ & $0.117_{\pm 0.178}$ & $0.027_{\pm 0.042}$ & $0.062_{\pm 0.103}$ & $0.041_{\pm 0.051}$ & $0.031_{\pm 0.040}$ & $0.021_{\pm 0.035}$ & $\boldsymbol{\underline{0.011}}_{\pm 0.021}$ \\
 & LN & & $\boldsymbol{0.014}_{\pm 0.027}$ & $0.231_{\pm 0.260}$ & $0.249_{\pm 0.476}$ & $0.025_{\pm 0.036}$ & $0.095_{\pm 0.131}$ & $0.023_{\pm 0.033}$ & $0.056_{\pm 0.098}$ & $0.037_{\pm 0.065}$ & $0.025_{\pm 0.035}$ & $0.019_{\pm 0.032}$ & $\boldsymbol{\underline{0.007}}_{\pm 0.015}$ \\
 & RF & & $0.020_{\pm 0.027}$ & $0.256_{\pm 0.293}$ & $0.184_{\pm 0.275}$ & $0.019_{\pm 0.024}$ & $0.127_{\pm 0.202}$ & $0.025_{\pm 0.037}$ & $0.064_{\pm 0.107}$ & $0.045_{\pm 0.049}$ & $0.031_{\pm 0.037}$ & $\boldsymbol{0.018}_{\pm 0.045}$ & $\boldsymbol{\underline{0.013}}_{\pm 0.027}$ \\
 & XGB & & $0.026_{\pm 0.044}$ & $0.272_{\pm 0.337}$ & $0.175_{\pm 0.242}$ & $0.028_{\pm 0.031}$ & $0.130_{\pm 0.210}$ & $0.034_{\pm 0.056}$ & $0.067_{\pm 0.116}$ & $0.041_{\pm 0.046}$ & $0.036_{\pm 0.050}$ & $\boldsymbol{0.026}_{\pm 0.032}$ & $\boldsymbol{\underline{0.014}}_{\pm 0.020}$ \\
\midrule
\multirow{4}{*}{Avg. ($\uparrow$)} & all & $0.899_{\pm 0.104}$ & $\boldsymbol{0.884}_{\pm 0.122}$ & $0.685_{\pm 0.285}$ & $0.747_{\pm 0.312}$ & $0.880_{\pm 0.119}$ & $0.807_{\pm 0.215}$ & $0.878_{\pm 0.129}$ & $0.853_{\pm 0.168}$ & $0.840_{\pm 0.129}$ & $0.875_{\pm 0.126}$ & $0.881_{\pm 0.113}$ & $\boldsymbol{\underline{0.889}}_{\pm 0.106}$ \\
 & LN & $0.891_{\pm 0.125}$ & $\boldsymbol{0.881}_{\pm 0.140}$ & $0.695_{\pm 0.272}$ & $0.719_{\pm 0.405}$ & $0.872_{\pm 0.145}$ & $0.818_{\pm 0.200}$ & $0.874_{\pm 0.143}$ & $0.851_{\pm 0.182}$ & $0.835_{\pm 0.163}$ & $0.872_{\pm 0.141}$ & $0.875_{\pm 0.135}$ & $\boldsymbol{\underline{0.885}}_{\pm 0.126}$ \\
 & RF & $0.901_{\pm 0.104}$ & $\boldsymbol{0.885}_{\pm 0.119}$ & $0.686_{\pm 0.289}$ & $0.757_{\pm 0.286}$ & $0.884_{\pm 0.109}$ & $0.801_{\pm 0.233}$ & $0.881_{\pm 0.127}$ & $0.852_{\pm 0.171}$ & $0.839_{\pm 0.124}$ & $0.875_{\pm 0.125}$ & $0.883_{\pm 0.105}$ & $\boldsymbol{\underline{0.888}}_{\pm 0.100}$ \\
 & XGB & $0.906_{\pm 0.094}$ & $\boldsymbol{0.886}_{\pm 0.122}$ & $0.673_{\pm 0.326}$ & $0.765_{\pm 0.262}$ & $0.883_{\pm 0.113}$ & $0.803_{\pm 0.236}$ & $0.880_{\pm 0.131}$ & $0.854_{\pm 0.171}$ & $0.848_{\pm 0.116}$ & $0.878_{\pm 0.127}$ & $0.885_{\pm 0.112}$ & $\boldsymbol{\underline{0.895}}_{\pm 0.101}$ \\
\bottomrule
\end{tabular}

    }}
    \vspace{-1em}
\end{table*}
\begin{table*}[t]
    \centering
    \caption{Averaged fidelity performance of different models on all 9 datasets. The best scores and the second best scores are highlighted in bold with and without underscore, respectively.}
    \label{tab:ml-fidelity-summary}
    \setlength{\tabcolsep}{1pt}
    \resizebox{\linewidth}{!}{

\begin{tabular}{lcccccccc>{\columncolor{red!10}}c>{\columncolor{red!10}}c>{\columncolor{red!10}}c}
\toprule
 & FD & CTAB+ & TTVAE & TabSyn & GReaT & RTF & TabMT & TabuLa & TTF-S & TTF-L & TTF-NM \\
\midrule
 Shape & $\boldsymbol{\underline{0.931}}_{\pm 0.047}$ & $0.890_{\pm 0.067}$ & $0.736_{\pm 0.150}$ & $\boldsymbol{0.925}_{\pm 0.052}$ & $0.881_{\pm 0.062}$ & $0.924_{\pm 0.067}$ & $0.919_{\pm 0.052}$ & $0.916_{\pm 0.076}$ & $0.910_{\pm 0.037}$ & $0.915_{\pm 0.042}$ & $0.894_{\pm 0.081}$ \\
 Trend & $\boldsymbol{\underline{0.922}}_{\pm 0.064}$ & $0.813_{\pm 0.122}$ & $0.678_{\pm 0.219}$ & $0.911_{\pm 0.084}$ & $0.862_{\pm 0.088}$ & $0.899_{\pm 0.091}$ & $0.912_{\pm 0.065}$ & $0.894_{\pm 0.085}$ & $0.901_{\pm 0.069}$ & $\boldsymbol{0.913}_{\pm 0.076}$ & $0.908_{\pm 0.098}$ \\
\bottomrule
\end{tabular}

    }
    \vspace{-1em}
\end{table*}

\paragraph{Ordinal Cross-Entropy Loss.}
Cross-entropy loss (CEL) is typically used as the optimization objective of transformers. However, the quantized quantile tokens retain an inherent ordinal property, where closer IDs correspond to closer values. This relationship is not captured by standard CEL, where all classes as equally distinct. To address this, we replace the vanilla CEL with a specialized ordinal CEL for quantile tokens. 
Formally, let the predicted logits for a token be $\mathbf{z} \in \mathbb{R}^V$, the probability after softmax be $\mathbf{p}\in[0,1]^V$ ($\|\mathbf{p}\|_1=1$), and the target label be $t \in \{1, \dots, V\}$.
We define \textit{ordinal cross-entropy loss} (OCEL) as a weighted version of CEL. These weights are applied to unnormalized probabilities in both the numerator and denominator, depending on the current- and target-class pairs. It contrasts with the classical weighted CEL, which weights the loss per class based solely on the target class and applies to the numerator only. Formally, we write it as Equation~\ref{eq:sce} (recall the classical CEL in Equation~\ref{eq:ce}), where $w_{ti}$ denotes the weight for $z_i$, where $t$ is the target class.

\begin{equation}
    \label{eq:ce}
    \mathcal{L}_{\text{ce}}(\mathbf{z}, t)=-\log\frac{e^{z_t}}{\sum_{i=1}^{V}e^{z_i}}=-\log p_t
\end{equation}

\begin{equation}
    \label{eq:sce}
    \mathcal{L}_{\text{oce}}(\mathbf{z}, t)=-\log\frac{w_{tt}e^{z_t}}{\sum_{i=1}^Vw_{ti}e^{z_i}}
\end{equation}
\begin{lemma}
    \label{thm:opt}
    Given $\mathbf{w}>\mathbf{0}$ independent of $\mathbf{z}$, OCEL is optimized when $z_t\to\infty$ and $z_i\to-\infty,\forall i\ne t$.
\end{lemma}
\vspace{-0.5em}
\begin{theorem}
    \label{thm:adj}
    Let $f(|t-i|)$ be a variant of the distance between current class $i$ and target class $t$ when $f$ is non-negative and monotonically increasing.
    Let the weighted sum of this variant of distances to target incurred by some logit $\mathbf{z}$ be $D(\mathbf{z})=\sum_{i=1}^Vp_iw_i$.
    If $w_{ti}=f(|t-i|)$ for some $f$, then
    for $\mathbf{z},\mathbf{z}^*$ with $\mathcal{L}_{\text{ce}}(\mathbf{z},t)=\mathcal{L}_{\text{ce}}(\mathbf{z}^*,t)$ and $D(\mathbf{z})<D(\mathbf{z}^*)$, we must have $\mathcal{L}_{\text{oce}}(\mathbf{z},t)<\mathcal{L}_{\text{oce}}(\mathbf{z}^*,t)$. 
\end{theorem}

Both Lemma~\ref{thm:opt} and Theorem~\ref{thm:adj} are intuitive, and we provide rigorous proofs in Appendix. Theorem~\ref{thm:adj} implies that the \ul{OCEL penalizes (value being smaller) classes closer to the target class less than farther ones}. 
Any definition of $w_{ti}=f(|t-i|)$ satisfying the monotonicity constraint works for OCEL technically.
In this paper, we define the weight $w_{ti}$ for $z_i$ to be obtained by Equation~\ref{eq:weight}.
\begin{equation}
    \label{eq:weight}
    w_{ti}=1+m-e^{-\frac{(t-i)^2}{(V\sigma)^2}}
\end{equation}
where $\sigma=0.005$ is a scaling factor of the distance, and $m=0.5$ is the minimum weight (used on $i=t$). Detailed explanation and analysis of it can be found in Appendix.

Note that not all tokens are part of the ordinal relation, so the OCEL is applied only to valid quantile tokens at the corresponding positions. To enable the model to learn to generate valid quantile tokens, we introduce an additional modified CEL to distinguish valid tokens from invalid ones, which complements OCEL. 
See Appendix for the exact overall loss formula.

\section{Experiments}
\label{sec:exp}

\subsection{Experimental Setup}

\begin{figure*}[t]
    \begin{minipage}[t]{0.75\linewidth}
        \centering
    
\includegraphics[width=\linewidth,trim={0 0.5em 0 0.2em},clip]{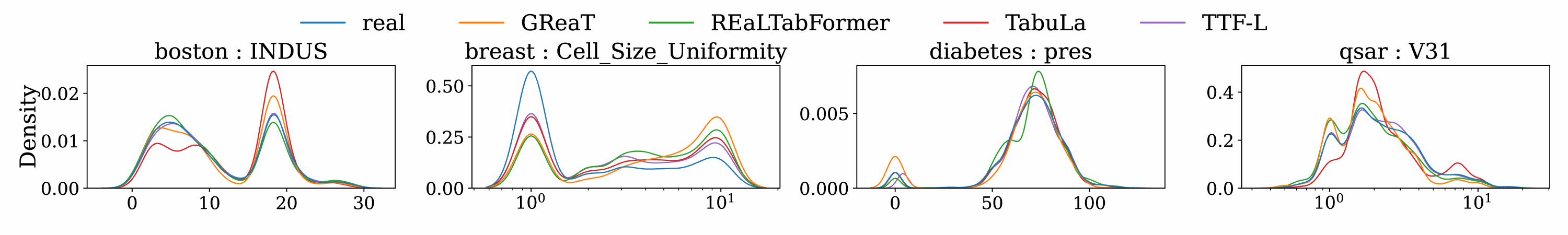}
    \caption{
    Marginal densities of representative multimodal continuous columns from baseline ART and TTF. All have a Distill-GPT2 backbone.
    }
    \label{fig:multimodal}
    
    \includegraphics[width=\linewidth]{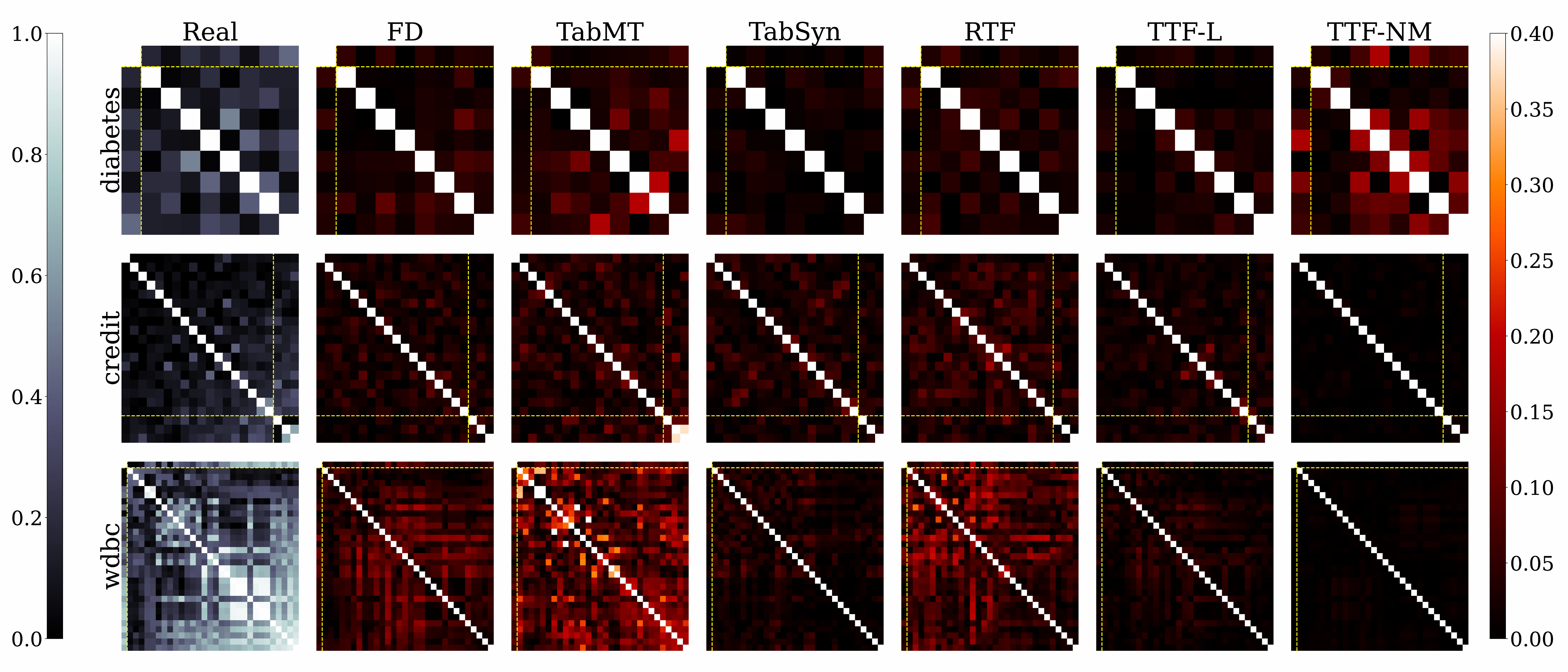}
    \caption{TTF versus top-4  baselines in pair-wise correlation. Real \ul{(absolute) correlation} values are presented on the left, and the \ul{absolute error} in correlation values in synthetic data from different models, capped at 0.4 for visibility, are shown at the right (\ul{the darker the better}). The left-top features (divided by yellow line) are categorical and the right-bottom are numeric.}
    \label{fig:corr}
    \end{minipage}%
    \hfill
    \begin{minipage}[t]{0.23\linewidth}
        \centering
    \vspace{-4.7em}
    \includegraphics[width=\linewidth,trim={18pt 20pt 18pt 20pt},clip]{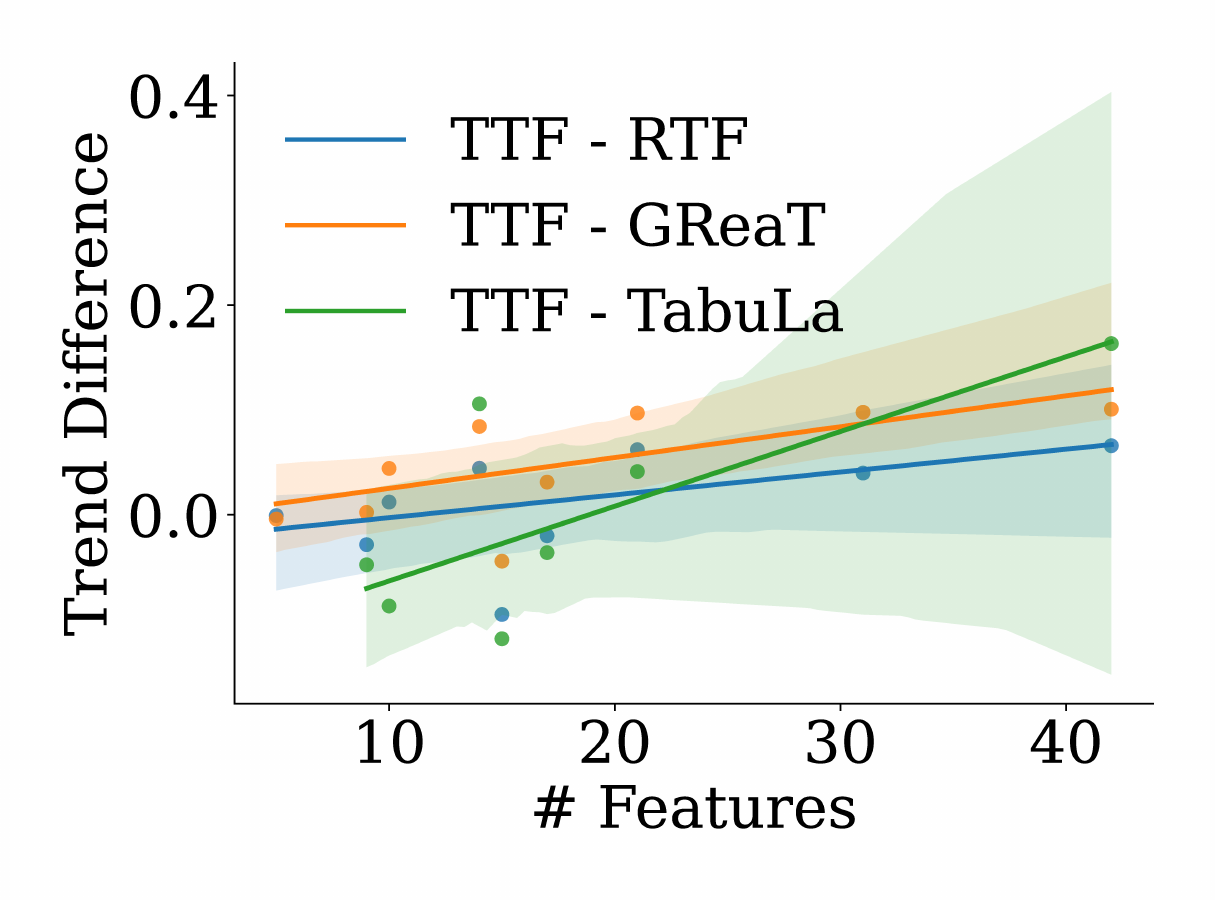}
    \caption{Feature count vs. Trend score improvement of TTF from ART baselines.}
    \label{fig:corr-trend}

    \vspace{0.7em}
    \includegraphics[width=\linewidth,trim={10pt 20pt 10pt 2em},clip]{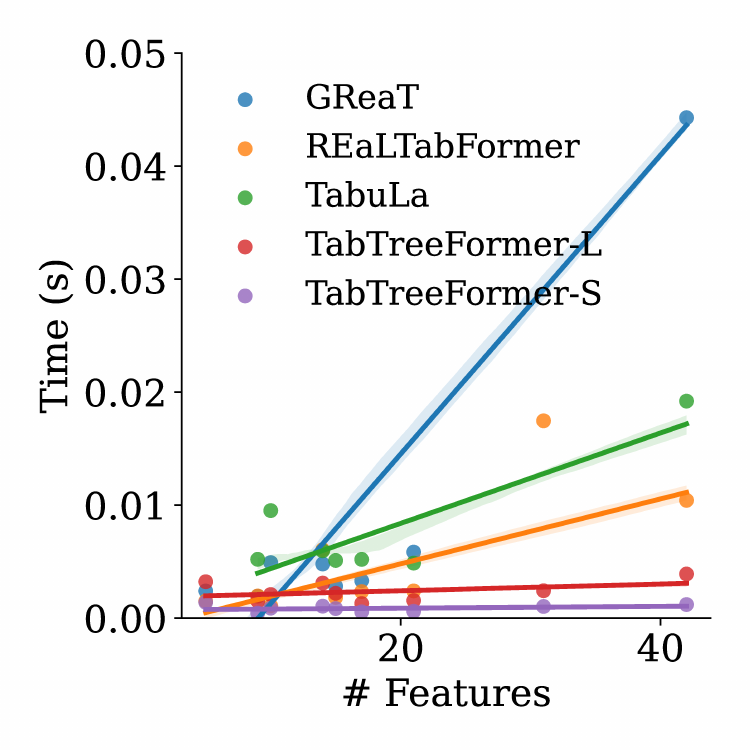}

    \caption{Number of features vs. average generation time per (valid) row of TTF and ART baselines.}
    \label{fig:gen-time}
    \end{minipage}
\end{figure*}

\label{sec:exp:quality}
\paragraph{Basic Setup.}
For all datasets, we split train and test sets in $4:1$, and repeat each experiment 3 times and report the results. 
We conduct all experiments on a machine equipped with 1 NVIDIA RTX 4090 and 20 CPU cores. Detailed environment setting can be found in Appendix.

\paragraph{\framework Configuration.}
We use  LightGBM~\citep{lgbm} whose hyperparameters are tuned by Optuna~\citep{optuna} as the tree-based model. We exploit Distill-GPT2~\citep{distilbert} as the transformer backbone, following the practice in prior works~\citep{great,realtabformer}. \frameworkname\space(TTF) is experimented with two major configurations: S (small) and L (large), with the number of trainable parameters of approximately 5M and 40M, respectively.
When privacy is not crucial concerns, we remove the masks of TTF-L, which we call TTF-NM (\textbf{n}o \textbf{m}ask).
In ablation study, we use \framework-S.
More implementation details can be found in Appendix.

\paragraph{Datasets.} Experiments are conducted on 9 datasets with diverse sizes and characteristics from OpenML~\citep{openml}: adult, bank, boston, breast, credit, diabetes, iris, qsar, and wdbc. 
Only credit and diabetes are used for ablation study. Details of datasets are summarized in Appendix.

\paragraph{Baseline Models.} We compare the performance of \framework against 8 SOTA methods in tabular data generation, including non-neural networks, GANs, VAEs, diffusion models, and auto-regressive transformers (ART) or masked transformers: Forest Diffusion (FD)~\citep{fdiff}, CTAB-GAN+ (CTAB+)~\citep{ctabganp}, TTVAE~\citep{ttvae}, TabSyn~\citep{tabsyn}, GReaT~\citep{great}, REaLTabFormer (RTF)~\citep{realtabformer}, TabMT~\citep{tabmt}, and TabuLa~\citep{tabula}. Justification of the baseline choices and implementation details are described in Appendix.

\begin{figure*}[t]
    \begin{minipage}[t]{0.35\linewidth}
        \centering
        \includegraphics[width=\linewidth,clip={0 2em 0 0}]{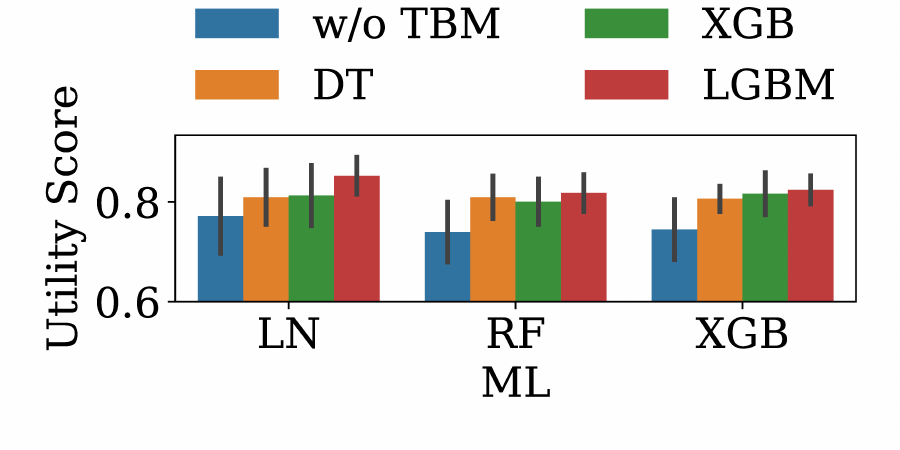}
        \vspace{-1.5em}
        \caption{Utility of different tree-based model in TTF. The $x$-axis shows the downstream model type.}
        \label{fig:tbm}
    \end{minipage}
    \hfill
    \begin{minipage}[t]{0.27\linewidth}
        \centering
        \includegraphics[width=0.9\linewidth]{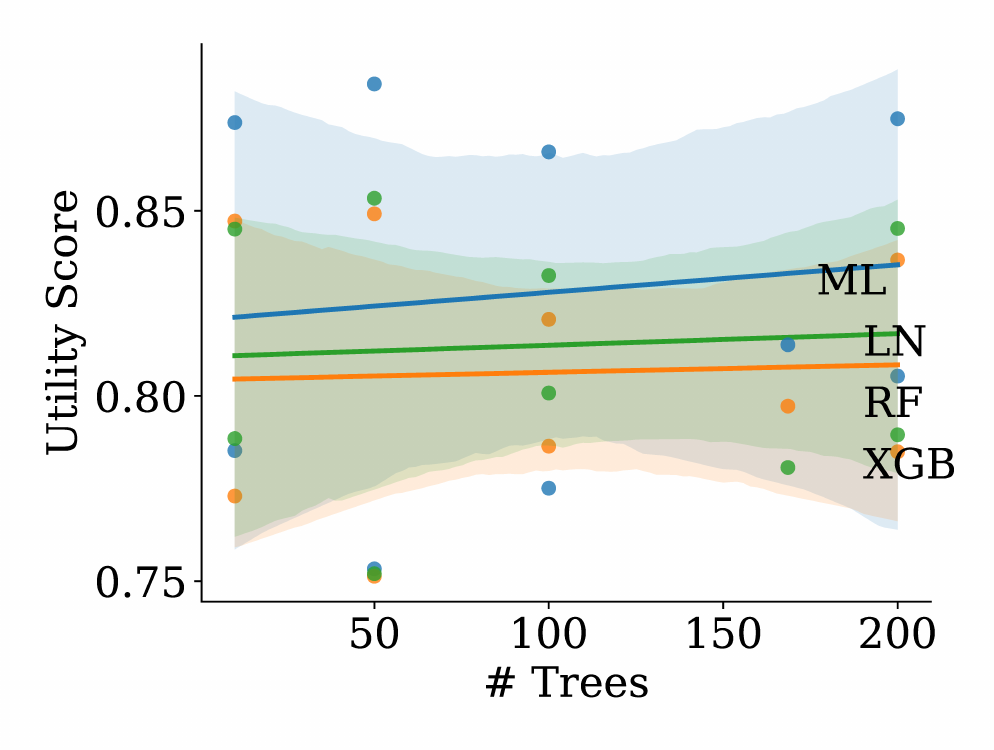}
        \vspace{-0.5em}
        \caption{Utility of different number of trees in the tree-based model in TTF.}
        \label{fig:ntrees}
    \end{minipage}
    \hfill
    \begin{minipage}[t]{0.35\linewidth}
    \centering
    \vspace{-9em}
    \captionsetup{type=table}
    \caption{Design component ablation. Average utility scores are reported. All experiments are based on variants of \framework-S.}
    \label{tab:abl-summary}
    \setlength{\tabcolsep}{2pt}
    \resizebox{0.95\linewidth}{!}{
    \begin{tabular}{lrrrr}
\toprule
ML & all & LN & RF & XGB \\
\midrule
TTF-S & 0.832 & 0.853 & 0.818 & 0.824 \\
\midrule
gen. w/o tree & 0.792 & 0.808 & 0.780 & 0.789 \\
w/o OCEL & 0.809 & 0.825 & 0.802 & 0.800 \\
w/o QE & 0.801 & 0.820 & 0.783 & 0.799 \\
\bottomrule
\end{tabular}
    }
    \end{minipage}
\end{figure*}
\paragraph{Metrics.} We evaluate \framework using utility, fidelity, privacy, and efficiency, which are standard metrics in tabular data generation~\citep{ctgan,ctabganp,great,tabsyn,tabdiff}. Detailed implementation of metrics can be found in Appendix.
\begin{itemize}[topsep=0pt, partopsep=0pt, itemsep=0pt, parsep=0pt,leftmargin=*]
    \item \textbf{Utility} exhibits the quality of synthetic data by testing its performance on downstream tasks. We evaluate based on the established Train-on-Synthetic, Test-on-Real (TSTR) machine learning efficacy (MLE) evaluation framework~\citep{ctgan}, namely, synthetic data generated based on real training data is tested on a hold-out real test set. Three models are used: logistic regression (e.g., linear--LN), random forest (RF), and XGBoost (XGB). Classification performance is evaluated by weighted AUC and regression is evaluated by $R^2$. For both metrics, \ul{higher scores indicate better performance.} MLE is the most widely used synthetic tabular data generation evaluation metric, which is reported in all baselines, 
    so we treat this as the \ul{core synthetic data quality score}.
    \item \textbf{Fidelity} shows the cosmetic discrepancy between the real and synthetic data. It is evaluated via ``Shape'' and ``Trend'' metrics ~\citep{sdmetrics,tabdiff}. ``Shape'' measures the similarity of marginal distribution density for each column, and ``Trend'' measures the fidelity in correlation between column pairs. \ul{Higher ``Shape'' and ``Trend'' values indicate better data fidelity.}
    \item \textbf{Privacy} is a crucial criterion when the synthetic data serves the use case of privacy-preserving data sharing. We use the distance to the closest record (DCR)~\citep{ctabgan} to evaluate the privacy of synthetic data. We compare the DCRs from synthetic data and from hold-out real (test) data to the real training data. Mann-Whitney U Test~\citep{mwu} is applied on these two sets of DCRs and the null hypothesis assumes the former DCRs are no smaller than the latter. \ul{Privacy is at risk of disclosure when $p$-values are smaller than 0.05.}
    \item \textbf{Efficiency} showcases the model sizes and computation time for training and generation of different models. \ul{Under comparable performance, smaller models and faster computation are favored.}
\end{itemize}

\subsection{Result Analysis}
\paragraph{Utility:}
Table~\ref{tab:ml-utility-summary} shows \framework-NM consistently outperforms all baselines on MLE, demonstrating its superior synthetic-data generation. The margin grows with stronger downstream models (XGB $>$ RF $>$ LN), underscoring \framework’s ability to capture complex relations via tree-based modeling. Results on individual datasets are detailed in Table~\ref{tab:ml-utility} in Appendix.

\paragraph{Fidelity} 
Table~\ref{tab:ml-fidelity-summary} shows the fidelity results.
\framework achieves comparable performance to baseline models in ``Shape'' metric, especially to ART-based baseline models. 
Meanwhile, Fig.~\ref{fig:multimodal} shows that \framework captures better (i.e., closer to real) multimodal distribution compared to other ARTs, which validates the effectiveness of our dual-quantization tokenizer design. 
\framework-L outperforms all neural network baselines in ``Trend'', implying the superior capability of \framework in learning inter-feature relations. 
We visualize the pair-wise correlation of some datasets in Fig.~\ref{fig:corr}.
While \framework-NM is able to capture the correlations on certain datasets perfectly, its performance variance across datasets is large.
In addition, \framework tends to show a more significant improvement from ART baselines with more features, as referred in Fig.~\ref{fig:corr-trend}. This demonstrates the advantage of introducing the inductive bias to address low-correlations so that models learn better by filtering out less correlated features, which is common for dataset with more features.

\begin{figure*}[t]
    \begin{minipage}[t]{0.32\linewidth}
        \includegraphics[width=0.85\linewidth]{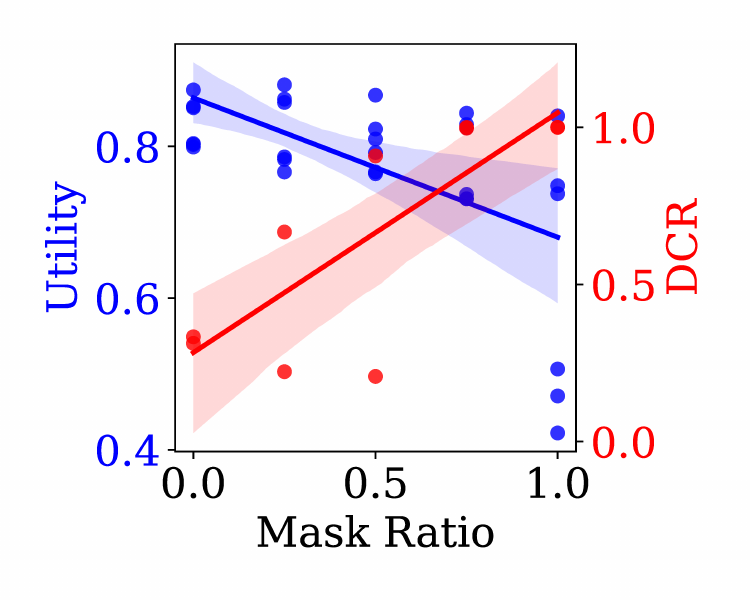}
        \caption{Quality \& privacy trend of different mask ratios.}
        \label{fig:mask-ratio}
    \end{minipage}
    \hfill
    \begin{minipage}[t]{0.32\linewidth}
        \includegraphics[width=0.85\linewidth]{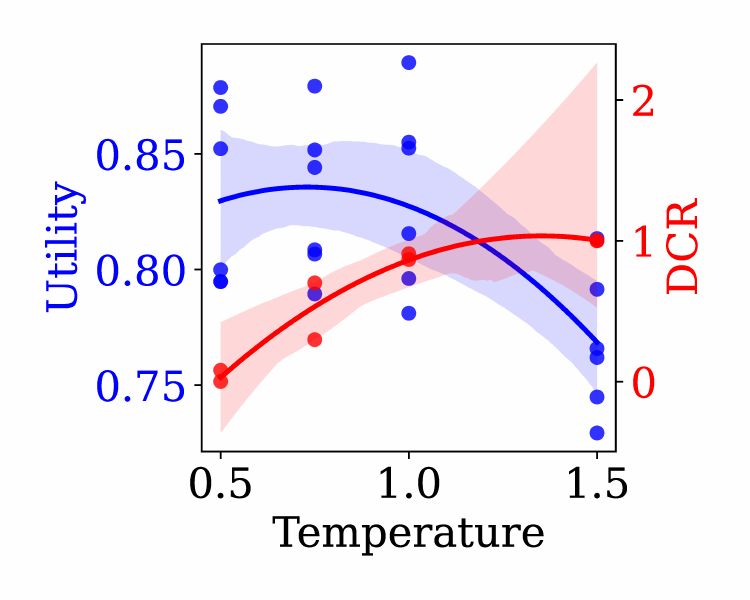}
        \caption{Quality \& privacy of trend of different temperatures.}
        \label{fig:temperature-quality}
    \end{minipage}
    \hfill
    \begin{minipage}[t]{0.32\linewidth}
        \includegraphics[width=0.85\linewidth,trim={0 0 0 2em},clip]{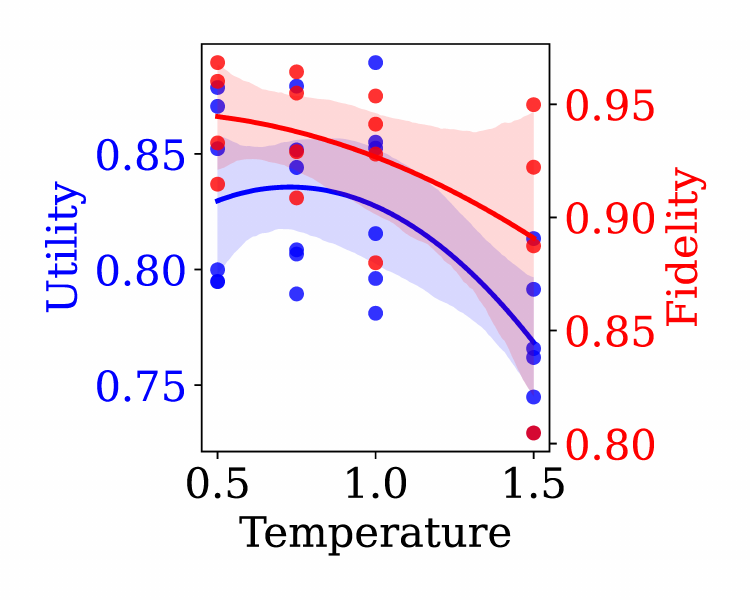}
        \caption{Utility \& fidelity of trend of different temperatures.}
        \label{fig:temperature}
    \end{minipage}
\end{figure*}

\paragraph{Privacy:}

\begin{table}[t]
  \centering
  \caption{%
    Privacy performances. The first two rows are the average and minimum $p$‐values,
    and the last row shows the \ul{number of violations} (i.e., $p<0.05$).
    The closer the $p$‐value to 0, the higher the risk of privacy leakage.
    Non‐zero violations are highlighted in red, indicating privacy risks.
  }
  \label{tab:dcr-summary}
  \setlength{\tabcolsep}{1pt}
  \resizebox{\linewidth}{!}{%
    \begin{tabular}{lcccccccc>{\columncolor{red!10}}c>{\columncolor{red!10}}c>{\columncolor{red!10}}c}
\toprule
 & FD & CTAB+ & TTVAE & TabSyn & GReaT & RTF & TabMT & TabuLa & TTF-S & TTF-L & TTF-NM \\
\midrule
Avg. $p$ & 0.380 & 1.000 & 0.904 & 0.559 & 0.843 & 0.394 & 0.865 & 0.251 & 0.933 & 0.745 & 0.417 \\
Min. $p$ & 0.000 & 1.000 & 0.382 & 0.029 & 0.082 & 0.000 & 0.259 & 0.000 & 0.662 & 0.214 & 0.000 \\
\# vio. & \textcolor{red}{\textbf{5}} & 0 & 0 & \textcolor{red}{\textbf{1}} & 0 & \textcolor{red}{\textbf{3}} & 0 & \textcolor{red}{\textbf{5}} & 0 & 0 & \textcolor{red}{\textbf{3}} \\
\bottomrule
\end{tabular}

  }
\end{table}

\begin{table}[t]
  \centering
  \caption{%
    Average computation time (s) of \frameworkname\space and  ART baselines.
    The best values are highlighted in \textbf{bold with underline}, the
    second best is in \textbf{bold without underline}.
  }
  \label{tab:train}
  \setlength{\tabcolsep}{2pt}
  \resizebox{0.8\linewidth}{!}{%
    \begin{tabular}{lccc>{\columncolor{red!10}}c>{\columncolor{red!10}}c}
\toprule
 & GReaT & RTF & TabuLa & TTF-S & TTF-L \\
\midrule
Train & 2470.344 & \textbf{359.950} & 718.260 & \textbf{\underline{316.743}} & 773.362 \\
Generate & 45.998 & 24.932 & 74.702 & \textbf{\underline{7.846}} & \textbf{19.718} \\
\bottomrule
\end{tabular}
  }
\end{table}

As summarized in Table~\ref{tab:dcr-summary}, \framework-S and L are privacy-resilient in all 9 datasets. In comparison, FD has privacy issues in 5 out of 9 datasets, which is consistent with our analysis in \textbf{Related Works} on its potential privacy issues.
As for ART baseline methods, all but GReaT demonstrates privacy risk, particularly TabuLa, which bears severe privacy leakage concerns in 5 out of 7 successfully run datasets. This result also verifies that the masking and early stopping of \framework's design successfully protects privacy.
Raw $p$-values per experiment are presented in Table~\ref{tab:dcr} in Appendix.

\paragraph{Efficiency:}

The efficiency of \framework is mainly demonstrated by its ability to achieve comparable performance with a significantly smaller model size. 
Although \framework-L is better than \framework-S in general, demonstrating the usefulness of a larger model, \framework-S achieves comparable utility with the best-performing ART baseline (REaLTabFormer, which has only marginally better MLE results than \framework-S), while the baseline models have more than 80M parameters. 

\paragraph{Efficiency:} 
\framework's computation time compared to all ART baselines is shown in Table~\ref{tab:train}. 
\framework consistently achieves faster generation.
Raw training and generation time can be found in Table~\ref{tab:time} in Appendix.

\subsection{Ablation Study}

A core design of \framework is its integration of a tree-based model (TBM). As shown in Fig.~\ref{fig:tbm}, having a tree-based model integrated is helpful. However, the exact tree-based model used does not seem to make a significant difference, though a subtle trend that multiple-tree models (XGBoost and LightGBM) are better than single-tree model (decision tree (DT)) can be observed.
Nevertheless, observing the effect of the number of trees on multiple-tree models, as shown in Fig.~\ref{fig:ntrees}, no obvious trend can be observed either, though there is still a general but very subtle trend that more trees result in better performance. 
we also evaluate the ablations of (1) tree-leaf tokens as prompt (2) OCEL and (3) quantile embedding (QE) to TabTreeFormer-S. The result is shown in Table~\ref{tab:abl-summary}. All the three components have a significant impact on improving the performance.
The mask ratio and temperature influence on \framework-NM is also evaluated. The result shows that the larger ratio is masked or the larger temperature, the better privacy and worse utility, this result is discussed at Appendix due to page limit.

\subsection{\framework-NM Settings.}
\label{app:result:nm}
Two core settings of NM different from L version is the mask ratio and temperature, so we also look at how the two parameters affect the performance, as shown in Fig.~\ref{fig:mask-ratio}-\ref{fig:temperature}. The utility and DCR changes with mask ratio and temperature demonstrate clearly the dilemma between synthetic data quality and privacy, and as expected, the larger ratio is masked or the larger temperature, the better privacy and worse utility.
Nevertheless, Fig.~\ref{fig:temperature} shows a slight different trend on performance of utility and fidelity on temperature. As we maximize the synthetic data quality mainly by utility, it is likely that the temperatures is not maximized with the fidelity, which explains the reason for \framework's score in fidelity (recall Table~\ref{tab:ml-fidelity-summary}) not being the best.

\section{Conclusion}
\label{sec:conclusion}

In this paper, we propose \framework, a hybrid tree-transformer based model for high-quality tabular data generation. It combines tree-based inductive biases with a dual-quantization tokenizer to model multimodal numeric distributions. Specialized loss and embeddings applied on quantized ordinal tokens. When evaluated on 9 datasets against 8 baselines, \framework achieves stronger balance on quality, privacy, and efficiency, and particularly excels in data utility when privacy and efficiency are not prioritized.

\bibliographystyle{apalike}
\bibliography{main}

\clearpage
\appendix
\section{Supplementary \framework Model Description and Details}
\label{app:model}
Section \textbf{End-to-end \framework Pipeline} provides the full end-to-end training and generation process, and Section \textbf{Tree-based Model} and \textbf{Transformer} provide detailed model setup of core components.

\subsection{End-to-end \framework Pipeline}
\label{app:model:algo}

As a supplementary of Fig.~\ref{fig:overview} on the overall process of \framework, and the textual description in Section \textbf{\frameworkname}, we provide the full end-to-end pseudocode in this section.
Algorithm~\ref{algo:train} and Algorithm~\ref{algo:generate} describe the end-to-end training and generation process of \framework respectively.
\begin{algorithm}[h]
   \caption{Training of \frameworkname}
   \label{algo:train}
   \begin{algorithmic}[1] 
   \footnotesize
      \STATE \textbf{Input:} Tabular dataset $\mathbf{X}$ ($n$ rows, $m_d$ discrete features, $m_c$ continuous features), $s_0$ shared initial steps, $s$ total training steps
      \STATE \textbf{Output:} Fitted tree-based model $\mathcal{T}$ and the \textit{leaf index matrices} $\mathbf{J}_1,\mathbf{J}_2$, data tokenizer $\mathcal{D}$, and transformers $\mathcal{G}_1,\mathcal{G}_2$
      \STATE $\mathcal{T} \gets \textsc{TuneAndFitTreeBasedModel}(\mathbf{X})$ \quad\quad$\blacktriangleright$ Let $T = \mathcal{T}.\textsc{n\_trees}$, $n_l = \mathcal{T}.\textsc{max\_n\_leaves}$\label{algo:line:nt}
      \STATE $\mathbf{J} \gets \mathcal{T}.\textsc{GetLeafIndex}(\mathbf{X})$ \quad\quad$\blacktriangleright$ Obtain leaf index matrix $\mathbf{J} \in \mathbb{N}^{n \times T}$
      \STATE $\mathcal{D} \gets \textsc{FitEncoder}(\mathbf{X})$ \quad\quad$\blacktriangleright$ $n_c = \mathcal{D}.\textsc{max\_n\_cat}$, $n_b = \mathcal{D}.\textsc{max\_n\_bin}$, $n_q = \mathcal{D}.\textsc{max\_n\_quant}$
      \STATE $\widetilde{\mathbf{X}} \gets \mathcal{D}.\textsc{Encode}(\mathbf{X})$ \quad\quad$\blacktriangleright$ Get cat/bin/quant IDs $\widetilde{\mathbf{X}} \in \mathbb{N}^{n \times (m_d + 2m_c)}$
      \STATE $\mathbf{Z} \gets \textsc{ToTokenIDs}([\mathbf{J}; \widetilde{\mathbf{X}}])$ \quad\quad$\blacktriangleright$ Convert these IDs to token IDs, add BOS, EOS
      \STATE $\mathcal{G} \gets \textsc{InitializeLM}(\text{n\_pos} = 2 + T + m_d + 2m_c, \text{n\_vocab} = 3 + n_l + n_c + n_b + n_q)$
      \\\quad\quad$\blacktriangleright$ Also initialize embedding layer on quantile tokens with ordinal relations
      \FOR{$i = 1, \dots, s_0$}
         \STATE $\mathbf{B} \gets \textsc{SampleBatch}(\mathbf{Z})$
         \STATE $\widetilde{\mathbf{B}} \gets \textsc{Mask}(\mathbf{B})$
         \STATE $\mathcal{G} \gets \textsc{UpdateGradient}(\text{input} = \widetilde{\mathbf{B}}, \text{target} = \mathbf{B})$
      \ENDFOR\quad\quad$\blacktriangleright$ Prepare some non-overfitted trained weights
      \STATE $[\mathbf{J}_1,\mathbf{J}_2], [\mathbf{Z}_1,\mathbf{Z}_2]\gets$\textsc{EvenlySplit}($\mathbf{J},\mathbf{Z}$)
      \STATE $\mathcal{G}_1,\mathcal{G}_2\gets$\textsc{MakeCopies}($\mathcal{G},2$)
      \FOR{$j=1,2$}
      \FOR{$i = s_0+1, \dots, s$}
         \STATE $\mathbf{B} \gets \textsc{SampleBatch}(\mathbf{Z}_j)$
         \STATE $\widetilde{\mathbf{B}} \gets \textsc{Mask}(\mathbf{B})$
         \STATE $\mathcal{G}_j \gets \textsc{UpdateGradient}(\text{input} = \widetilde{\mathbf{B}}, \text{target} = \mathbf{B})$
         \IF{\textsc{ValidateAtStep}($i$)}
            \STATE $\ell\gets\textsc{ComputeLoss}(\text{input}=\textsc{Mask}(\mathbf{Z}_{1-j}),\text{target}=\mathbf{Z}_{1-j})$\quad\quad$\blacktriangleright$ Compute validation loss
            \IF{\textsc{MetEarlyStopCriterion}($\ell$)}
            \STATE $i\gets s$\quad\quad$\blacktriangleright$ Early stop
            \ENDIF
         \ENDIF
      \ENDFOR
      \ENDFOR
   \end{algorithmic}
\end{algorithm}

\subsection{Tree-based Model}
\label{app:model:tbm}

\subsubsection{Hyperparameter Tuning}

For the tree-based classifier or regressor, we use Optuna~\citep{optuna} to tune the hyperparameters of LightGBM~\citep{lgbm} using the training data. If the dataset has a designated target column, which is the case for all our experimented datasets, the model is trained to predict that column. Otherwise, a random column can be the target.
Then, fit the model with the selected hyperparameters for \framework.

The hyperparameter space explored is
\begin{itemize}
    \item \textbf{Learning Rate}: Logarithmic scale float in $[0.01, 0.3]$;
    \item \textbf{Number of Estimators}: Integers in $[50,250]$ step 50;
    \item \textbf{Maximum Depth}: Integers in $[3,10]$;
    \item \textbf{Number of Leaves}: Integers in $[20,100]$ step 5;
    \item \textbf{Minimum Size per Leaf}: Integers in $[10,50]$ step 5;
    \item \textbf{Feature Fraction}: Float in $[0.6, 1.0]$;
    \item \textbf{Bagging Fraction}: Float in $[0.6, 1.0]$;
    \item \textbf{L1 Regularization}: Float in $[0.0,10.0]$;
    \item \textbf{L2 Regularization}: Float in $[0.0,10.0]$.
\end{itemize}
\begin{algorithm}[t]
   \caption{Tabular data generation using \frameworkname}
   \label{algo:generate}
   \begin{algorithmic}[1] 
   \footnotesize
      \STATE \textbf{Input:} Leaf index matrices $\mathbf{J}_1,\mathbf{J}_2$, data tokenizer $\mathcal{D}$, transformer-based language models $\mathcal{G}$, and number of rows generate $n'$
      \STATE $n_1,n_2\gets\lfloor n'/2\rfloor,\lceil n'/2\rceil$\quad\quad$\blacktriangleright$ Assign number of rows for each network
      \STATE \textbf{Output:} Synthetic tabular dataset $\mathbf{X}'$
      \FOR{$j=1,2$}
          \STATE $\mathbf{J}_j'\gets\textsc{Sample}(\mathbf{J}_j,n_j)$
          \STATE $\mathbf{Z}_j'\gets\textsc{Mask}(\textsc{ToTokenIDs}(\mathbf{J}_j'))$\quad\quad$\blacktriangleright$ Prepare tokens as prompts
          \STATE $\widetilde{\mathbf{X}}_j'\gets\mathcal{G}.\textsc{Generate}(\text{partial\_input}=\mathbf{Z}_j')$
          \STATE $\mathbf{X}_j'\gets\mathcal{D}.\textsc{Decode}(\widetilde{\mathbf{X}}_j')$
      \ENDFOR
      \STATE $\mathbf{X}'\gets\textsc{Combine}(\mathbf{X}_1',\mathbf{X}_2')$
   \end{algorithmic}
\end{algorithm}

Optimization was performed over 50 trials using 3-fold cross-validation, performance evaluated based on weighted F1 for classification and (negative) mean squared error for regression.

\subsubsection{Leaf Index Matrix}

In the subsequent steps of \frameworkname, leaf index matrix $\mathbf{J}$ of the real data $\mathbf{X}$ is computed. The indexing order of trees and their leaves used in the subsequent steps are directly taken from the fitted tree-based model. The leaf index matrix is essentially the result of \texttt{.apply} function for most tree-based models in \texttt{sklearn}.

\subsection{Tokenization}

For K-Means and quantile quantization, we use \verb|sklearn.preprocessing.KBinsDiscretizer| with strategy ``kmeans'' and ``quantile'' respectively. 
For categorical values, we use \verb|sklearn.preprocessing.OrdinalEncoder|.

In cases where the number of distinct values expected to be kept is way larger than typical natural language models' vocabulary size, two or more consecutive quantile quantizers may be used, where subsequent quantizers are used to quantize the values in the same quantile as suggested by previous quantizers. Thus, with $q$ quantizers, a total of $Q^q$ values can be represented, which is sufficient to cover most practical cases. In this paper, we present $Q=1000$ as sufficient precision for simplicity.

\subsection{Transformer}
\label{app:model:transformer}
\subsubsection{Model Configuration}

For the language model, we use Distill-GPT2~\citep{distilbert}, based on the GPT2 architecture~\citep{gpt2}, with several modifications. The vocabulary size and maximum sequence length are adjusted based on the dataset, and the values are likely much lower than the requirement for natural language models. For example, if $n_l=200,n_c=20,n_b=K=10,n_q=Q=1000$, then $V=1233$, and if $T=200,m_d=10,m_c=10$, then $L=222$, while natural language models typically have $V>30000,L>1000$.

In this paper, we present 3 versions of \frameworkname\space of different sizes. Table~\ref{tab:model-size} shows their details. The model sizes are obtained for diabetes~\citep{diabetes-data} dataset. Actual values may vary based on the fitted tree-based model and dataset.

\begin{table}[t]
    \centering
    \caption{Hyperparameter configuration of \frameworkname\space of different sizes.}
    \label{tab:model-size}
    \begin{tabular}{l|cc}
        \toprule
        \frameworkname & S & L \\
        \midrule
        Approx. \#Params (in M) & 5 & 40 \\
        \midrule
        hidden state dimensions & 256 & 768\\
        inner feed-forwared layer dimensions & 1024 & 3072 \\
        number of attention heads per layer & 8 & 12 \\
        \bottomrule
    \end{tabular}
\end{table}

\subsubsection{Token Sequence Construction}

The tokens in Table~\ref{tab:tokens} are conceptual tokens. When converting the leaf, category, K-Means bin, and quantile IDs into token IDs in actual implementation, we do not explicitly make the textual tokens, but instead, directly add corresponding token type's offset.

\subsubsection{Masking}

Instead of using a fixed mask ratio, we apply varying mask ratios to enhance the diversity of training data. Tree-related tokens (yellow section in Fig.~\ref{fig:overview}) are masked with a ratio sampled uniformly from $[0.5,0.75]$, while actual value tokens (blue section in Fig.~\ref{fig:overview}) are masked with a ratio sampled uniformly from $[0.25,0.5]$.
In the no-mask (NM) version, we set both mask ratios to be constant 0.

To ensure that the quantile token is always masked if its corresponding K-Means bin token is masked within the same numeric column, we first generate a random mask. If a violation is detected (i.e., the K-Means bin token is unmasked while the quantile token is masked), we swap the mask values of these two positions. This approach ensures that masks are generated both efficiently and accurately.

\subsubsection{Training Configuration}
The reduced model size allows for a significantly larger batch size compared to typical natural language models. We use a batch size of 128 per device, reduce by half in case of CUDA OOM\footnote{Among all $10\times3=30$ experiments, \frameworkname-L hits CUDA OOM in one experiment, and no other \frameworkname\space experiments had memory issues with this batch size.}, and train for up to 5,000 steps ($s=5,000$ in Algorithm~\ref{algo:train}) with a learning rate of $5 \times 10^{-4}$ in FP16 precision, utilizing the Hugging Face Trainer.
Early stopping may be triggered by validation loss with a patience of 3 every 100 steps.
Patience is set to 100 in the NM version.
The initial training on both splits (Line 9-13 in Algorithm~\ref{algo:train}) is trained with the number of steps ($s_0$) set to the smaller of the number of steps required for 20 epochs and 1/10 of $s$.

\subsubsection{Model Inference}

Tree-leaf tokens from real data after masking are used as prompts for generation. 
Values of categorical columns are more likely to be memorized than numeric values by our ordinal learning method.
Therefore, we apply a higher temperature to the tokens for categorical values and a lower one to the tokens for numeric values.
In particular, we set the temperature for categorical tokens to be 2.0 and numeric tokens, including bin and qunatile tokens, to be 1.0.
In the NM version, we set the temperatures to be 0.2 and 0.1 respectively.

Recall that a set of valid tokens can be pre-determined at each position. Formally, let $p_i$ be the corresponding feature index for the $i$-th token from data tokenizer, and $R_i$ indicates the type of this token. For example, in Fig.~\ref{fig:overview}, $p_1=1,p_2=1,p_3=2,R_1=\text{bin},R_2=\text{quant},R_3=\text{cat}$. Then, let $\mathcal{V}_i$ be the set of tokens allowed at position $i\in\{1,2,\dots,L\}$, we have Equation~\ref{eq:range-vocab}.
\begin{equation}
  \label{eq:range-vocab}
  \resizebox{\columnwidth}{!}{$
    \mathcal{V}_i
    =\begin{cases}
      \{\mathrm{[BOS]}\}, & i=1,\\
      \{\mathrm{[leaf1]},\ldots,\mathrm{[leaf}l_{i-1}\mathrm{]}\}, & 2\le i\le T+1,\\
      \{\mathrm{[bin1]},\ldots,\mathrm{[bin}b_{p_{i-T-1}}\mathrm{]}\}, & T+2\le i\le L-1,\ R_{i-T-1}=\mathrm{bin},\\
      \{\mathrm{[quant1]},\ldots,\mathrm{[quant}q_{p_{i-T-1}}\mathrm{]}\}, & T+2\le i\le L-1,\ R_{i-T-1}=\mathrm{quant},\\
      \{\mathrm{[cat1]},\ldots,\mathrm{[cat}b_{p_{i-T-1}}\mathrm{]}\}, & T+2\le i\le L-1,\ R_{i-T-1}=\mathrm{cat},\\
      \{\mathrm{[EOS]}\}, & i=L
    \end{cases}
  $}
\end{equation}
Valid tokens for a specific position is enforced by setting the logits of invalid tokens at the position to negative infinity before sampling.

\section{Supplementary Quantile Embeddings Explanation and Analysis}
\label{app:emb}

\subsection{Preliminaries: Function-Generated Positional Embeddings}

While positional embeddings can also be a learned embedding like word tokens, \citet{transformer} found that its performance is not necessarily significantly better than function-generated positional embeddings.
Moreover, function-generated embeddings allow the model to extrapolate positional information in sentences longer than those in the training corpora.

The typical functions to generate positional embeddings written in a similar format as Equation~\ref{eq:bin-emb} (expressed in $i$ and $d$) are shown in Equation~\ref{eq:pos-emb}~\citep{transformer}.
\begin{align}
    PE_{id}=\begin{cases}
        \sin\left(\frac{i}{10000^{\frac{d}{D}}}\right)&\text{if }2\mid i\\
        \cos\left(\frac{i}{10000^{\frac{d-1}{D}}}\right)&\text{if }2\nmid i\\
    \end{cases}
    \label{eq:pos-emb}
\end{align}
The functions are essentially sine and cosine functions, where the frequency is controled by the dimension indices.
Using these periodic functions is particularly suitable for positional embeddings, which focus more on relative differences than absolute differences.

\subsection{Choice of the Core Generator Function for Quantile Embeddings}
Periodic functions are proper generator functions for positional embeddings because of the relative nature of positions.
Similarly, due to the absolute and monotonic nature of quantile values, monotonic functions are more suitable for the generator functions for quantile embeddings.

The simplest of monotonic functions are linear functions, but they suffer from the following problems:
\begin{itemize}
    \item To obtain different linear functions at different embedding dimensions, the slopes and intercepts need to be changed. However, in order for the embedded values to fall in a reasonable range (no extremely large or small values), the choice of proper slopes and intercepts becomes very limited, which hurts the diversity between different embedding dimensions. This significantly downgrades the effect of having different embedding dimensions.
    \item Having linear functions in all dimensions means that different dimensions are in strict linear relations to one another. Note that transformers are heavily dependent on linear projections. In particular, the input embeddings are linearly projected to $Q,K,V$ vectors of the first attention layer. Therefore, universal linear relations among all embedded dimensions would degenerate the power of the first few layers of the network, which is an undesired outcome.
\end{itemize}

Therefore, besides being monotonic, we also require the generator function to be non-linear and have a controled range (so ideally bounded).
In this paper, we choose sigmoid, a simple function that satisfies the non-linear monotonic and bounded requirements.
Nevertheless, we also note that other functions satisfying the requirements may also be used, such as tanh, softsign, and piece-wise linear functions.

\subsection{Construction of Quantile Embeddings}
\begin{figure}
\centering
\includegraphics[width=\linewidth]{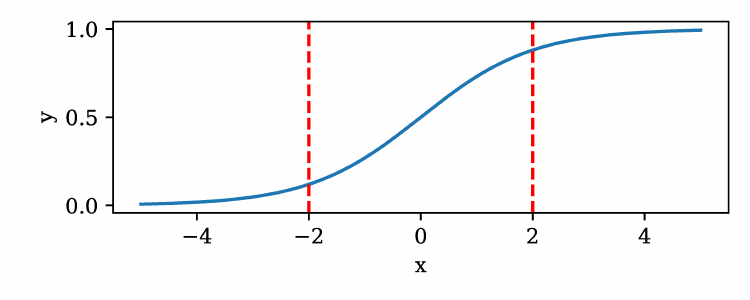}
\caption{Sigmoid function $y=\mathrm{sigmoid}(x)$. The red vertical dashed lines highlight the range $x\in[-2,2]$ where the changes of values are considered non-trivial.}
\label{fig:sigmoid}
\end{figure}

For the standard sigmoid function $\mathrm{sigmoid}(x)=\frac{1}{1+e^{-x}}$, we consider the range $[-2, 2]$ as the core range of $x$ that has non-trivial changes of values (see Fig.~\ref{fig:sigmoid}).
To project the total $Q$ quantile tokens on this core range, we need to scale and shift the quantile ID $i$ to $4\left(\frac{i}{Q}-\frac{1}{2}\right)$ (recall Equation~\ref{eq:bin-emb}).

The scale factor $S_d$ (recall Equation~\ref{eq:bin-scale}) controls how the core range is distributed over the total $Q$ quantile tokens. 
For a given scale factor value $s\in\mathbb{N}^+$, we assign $2s$ embedded dimensions with different offsets.
More dimensions are assigned to larger scale factors because of the smaller number of tokens having transformed input values falling in the core range of the $\mathrm{sigmoid}$ function.
Thus, the vector of $S_d$ for $d=0,1,\dots,D-1$ is constructed by $\verb|repeat(arange(n)+1, 2*(arange(n)+1))[:D]|$ of a large enough $\verb|n|$.

\begin{proposition}
    Equation~\ref{eq:bin-scale} is a closed form of the scale factor construction above ($\verb|repeat(arange(n)+1, 2*(arange(n)+1))[:D]|$).
\end{proposition}
\begin{proof}
    The construction means that for scale factor value $s\in\mathbb{N}^+$ starting from 1, the same value repeats $2s$ times. Let $S_d'$ be the value by the construction, and $\mathcal{J}(s)=\{d\mid S_d'=s\}$, then $\mathcal{J}(s)$ must be a set of consecutive integers, and

    \begin{align}
        \min\mathcal{J}(s)&=\sum_{i=1}^{s-1}2s=(1+s-1)(s-1)=s^2-s\label{eq:minj}\\
        \max\mathcal{J}(s)&=\min\mathcal{J}(s)+2s-1=s^2+s-1
    \end{align}

    Let $\hat{S}_d=\frac{1+\sqrt{1+4d}}{2}$, then $S_d=\lfloor\hat{S}_d\rfloor$.
    Then, by $s\ge1$,
\begin{align}
\hat{S}_{\min\mathcal{J}(s)}
&= \hat{S}_{s^2-s}
  \notag\\
&= \frac{1+\sqrt{1+4\,(s^2-s)}}{2}
  \notag\\
&= \frac{1+\sqrt{(2s-1)^2}}{2}
  \notag\\
&= \frac{1+|2s-1|}{2}
  \notag\\
&= s, 
  \tag{3a}\\
\hat{S}_{\max\mathcal{J}(s)}
&= \hat{S}_{s^2+s-1}
  \notag\\
&= \frac{1+\sqrt{1+4\,(s^2+s-1)}}{2}
  \notag\\
&< \frac{1+\sqrt{(2s+1)^2}}{2}
  \notag\\
&= \frac{1+|2s+1|}{2}
  \notag\\
&= s+1.
  \tag{3b}
\end{align}

    Note that $\hat{S}_d$ is monotonically increasing with respect to $d$, then we must have
    \begin{align}
        \hat{S}_d\in[s,s+1),\,\forall d\in\mathcal{J}(s)
    \end{align}

    This concludes to
    \begin{align}
        S_d=\lfloor\hat{S}_d\rfloor=s
    \end{align}

    Given $s\in\mathbb{N}^+$, then $S_d\in\mathbb{N}^+$.
\end{proof}

For embedding dimensions corresponding to the same scale factor $S_d$, we apply evenly distributed offsets in $[-2S_d,2S_d]$, where the constant factors $-2,2$ also come from the core range $x\in[-2,2]$. Thus, the vector of offsets $O_d$ for $d=0,1,\dots,D-1$ is constructed by $concat(linspace(-2*x,2*x,2*x)$ for $x$ in $(arange(n)+1))[:D]|$.

\begin{proposition}
    Equation~\ref{eq:bin-offset} is a closed form of the offset construction above ($concat(linspace(-2*x,2*x,2*x)$ for $x$ in $(arange(n)+1))[:D]|$).
\end{proposition}
\begin{proof}
    The \verb|x| in the construction is $s$ in the previous proof.
    Let the $i$-th offset value with $S_d=s$ be $o_{si}$ by the construction. Then we have $i\in[0,2s)$ and 
    \begin{align}
  o_{si}
  &= -2s + i\cdot\frac{2s - (-2s)}{2s-1}
  \notag\\
  &= \frac{-4s^2 + 2s + 4si}{2s-1}
  \notag\\
  &= \frac{2s\,\bigl(-2s + (2i+1)\bigr)}{2s-1}
  \label{eq:osi}
\end{align}

    Let $O'_d$ be the $d$-th value in the constructed offset value, regardless of scale factor, then substitute with Equation~\ref{eq:osi}, \ref{eq:minj} and \ref{eq:bin-offset},

    \begin{align}
        O_d'=o_{S_d(d-\min\mathcal{J}(S_d))}&=\frac{2S_d(-2S_d+(2(d-S_d^2+S_d)+1))}{2S_d-1}\\
        &=\frac{-4S_d^3+(4d+2)S_d}{2S_d-1}=O_d
    \end{align}

    Furthermore, $o_{si}\in[-2s,2s]$, so $O_d\in[-2S_d,2S_d]$.
\end{proof}

Applying the sigmoid function on the quantile IDs with the scale factors and offsets, we obtain the quantile embedding value generators in Equation~\ref{eq:bin-emb}.

\subsection{Comparison between Quantile Embeddings and Positional Embeddings}

Table~\ref{tab:emb-comp} shows a summary of the comparison between the proposed quantile embeddings and the typical function-generated positional embeddings.
The two function-generated embeddings are also visualized in Fig.~\ref{fig:comp-emb}.

\begin{table}[t]
    \centering
    \caption{Comparison between the proposed function-generated quantile and typical function-generated positional embeddings for transformers~\citep{transformer}.}
    \label{tab:emb-comp}
    \resizebox{\linewidth}{!}{
    \begin{tabular}{p{0.23\linewidth}p{0.45\linewidth}p{0.4\linewidth}}
        \toprule
         & Quantile Embeddings & Positional Embeddings \\
        \midrule
        Core function & Sigmoid & Trigonometric \\
        Core relation & Absolute ordinal value & Relative distance \\
        Function property & Non-linear and monotonic, range-controled & Periodic, range-controled \\
        Dimensions differ in & Scale factor and offset & Frequency and function (sin vs. cos) \\
        \bottomrule
    \end{tabular}
    }
\end{table}

\begin{figure}[t]
    \centering
    \begin{minipage}{0.48\linewidth}
        \centering
        \includegraphics[width=0.95\linewidth]{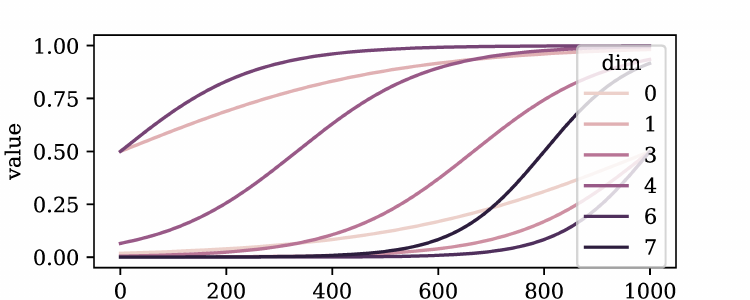}
    \end{minipage}%
    \hfill
    \begin{minipage}{0.48\linewidth}
        \centering
        \includegraphics[width=0.95\linewidth]{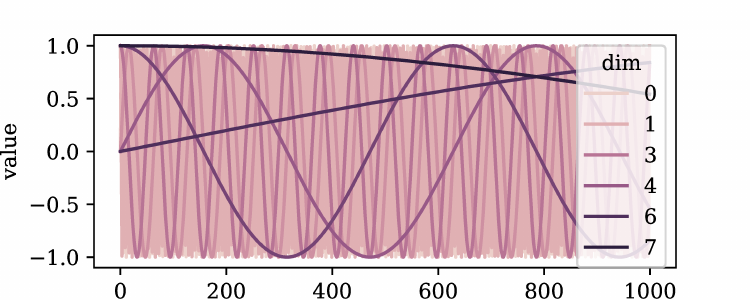}
    \end{minipage}
    
    \begin{minipage}{0.48\linewidth}
        \centering
        \includegraphics[width=0.95\linewidth]{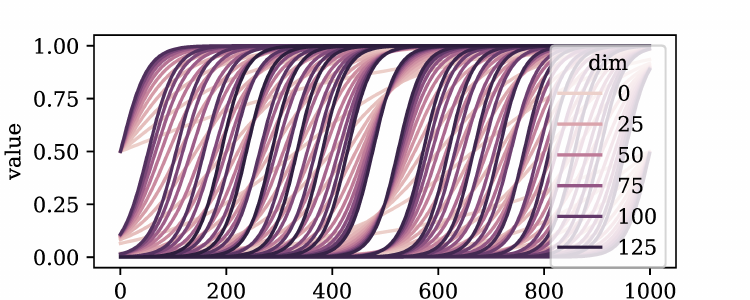}
    \end{minipage}%
    \hfill
    \begin{minipage}{0.48\linewidth}
        \centering
        \includegraphics[width=0.95\linewidth]{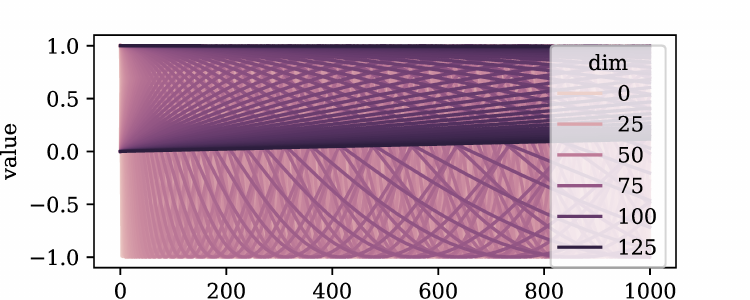}
    \end{minipage}

    \caption{Visualized comparison of function-generated quantile and positional embeddings. The $x$-axis represents $i\in\{0,1,\dots,Q-1\}$ (where $Q=1000$) and $y$-axis represents the corresponding embedded values. Different colors of lines represent different embedded dimensions. \textbf{Left:} quantile embeddings (Equation~\ref{eq:bin-emb}), \textbf{Right:} positional embeddings (Equation~\ref{eq:pos-emb}), \textbf{Top:} $D=8$, \textbf{Bottom:} $D=128$.
    }
    \label{fig:comp-emb}
\end{figure}

\subsection{Proof of Theorem~\ref{thm:emb}}
\label{app:emb:proof}

\begin{proof}
    For simplicity of symbols, let $q_{id}=QE_{id}$.
    By Equation~\ref{eq:bin-emb}, $q_{id}=\mathrm{sigmoid}(k_di+b_d)$ for some $k_d\in\mathbb{R}^+,b_d\in\mathbb{R}$.
    Note that $\mathrm{sigmoid}$ is strictly monotonically increasing, and the linear function $k_di+b_d$ with $k_d>0$ is also strictly monotonically increasing concerning $i$.
    Thus, $q_{id}$ is strictly monotonically increasing with respect to $i$.

    If $|i-j|<|i-k|$, then $|q_{id}-q_{jd}|<|q_{id}-q_{kd}|,\forall d\in\{0,1,\dots,D-1\}$, by the monotonicity of $q_{id}$ with respect to $i$ and triangle inequality.
    Then, by the strictly increasing monotonicity of $f(x)=x^p$ and $f(x)=x^{\frac{1}{p}}$ for $p\ge1$ on $x\in\mathbb{R}^+\cup\{0\}$ and by the monotonicity of the sum of monotonic functions, we must have
    \begin{align}
\|\mathbf{q}_i-\mathbf{q}_j\|_p
  &= \Bigl(\sum_{d=0}^{D-1}\lvert q_{id}-q_{jd}\rvert^p\Bigr)^{\frac{1}{p}}
  \notag\\
  &< \Bigl(\sum_{d=0}^{D-1}\lvert q_{id}-q_{kd}\rvert^p\Bigr)^{\frac{1}{p}}
  \notag\\
  &= \|\mathbf{q}_i-\mathbf{q}_k\|_p
  \label{eq:qp-distance-ineq}
\end{align}
    and by the monotonicity of the maximum of monotonic functions, we must have
    \begin{align}
\|\mathbf{q}_i-\mathbf{q}_j\|_\infty
  &= \max_{0\le d<D}\lvert q_{id}-q_{jd}\rvert
  \notag\\
  &< \max_{0\le d<D}\lvert q_{id}-q_{kd}\rvert
  \notag\\
  &= \|\mathbf{q}_i-\mathbf{q}_k\|_\infty
  \label{eq:qp-infty-ineq}
\end{align}
    Therefore, if $|i-j|<|i-k|$, then $\|\mathbf{q}_i-\mathbf{q}_j\|_p<\|\mathbf{q}_i-\mathbf{q}_k\|_p$.

    In the inverse direction, we prove the contrapositive. If $|i-j|\ge|i-k|$, then $|q_{id}-q_{jd}|\ge|q_{id}-q_{kd}|,\forall d\in\{0,1,\dots,D-1\}$. Next, by similar monotonic arguments above, we must have
    \begin{align}
    \|\mathbf{q}_i-\mathbf{q}_j\|_p
      &= \Bigl(\sum_{d=0}^{D-1}\lvert q_{id}-q_{jd}\rvert^p\Bigr)^{\tfrac{1}{p}}
      \notag\\
      &\ge \Bigl(\sum_{d=0}^{D-1}\lvert q_{id}-q_{kd}\rvert^p\Bigr)^{\tfrac{1}{p}}
      \notag\\
      &= \|\mathbf{q}_i-\mathbf{q}_k\|_p
      \label{eq:qp-pnorm-ge-ineq}
    \end{align}
    and
    \begin{align}
\|\mathbf{q}_i-\mathbf{q}_j\|_\infty
  &= \max_{0\le d<D}\lvert q_{id}-q_{jd}\rvert
  \notag\\
  &\ge \max_{0\le d<D}\lvert q_{id}-q_{kd}\rvert
  \notag\\
  &= \|\mathbf{q}_i-\mathbf{q}_k\|_\infty
  \label{eq:qp-infty-ge-ineq}
\end{align}

    The above shows that if $|i-j|\ge|i-k|$, then $\|\mathbf{q}_i-\mathbf{q}_j\|_p\ge\|\mathbf{q}_i-\mathbf{q}_k\|_p$.
    Therefore, if $\|\mathbf{q}_i-\mathbf{q}_j\|_p<\|\mathbf{q}_i-\mathbf{q}_k\|_p$, then $|i-j|<|i-k|$.

    Summarizing the above, we have $|i-j|<|i-k|$ if and only if $\|\mathbf{q}_i-\mathbf{q}_j\|_p<\|\mathbf{q}_i-\mathbf{q}_k\|_p$, and the conclusion holds for any $p$.
\end{proof}

\begin{remark}
    According to the proof, any strictly monotonically increasing core function, not restricted to sigmoid, makes Theorem~\ref{thm:emb} hold.
\end{remark}
\begin{remark}
    Any monotonically increasing core function, including non-strictly monotonic ones, makes the ``if'' direction of Theorem~\ref{thm:emb} hold, with no guarantee of the ``only if`` direction (the two $p$-norm distances can be equal).
\end{remark}

\section{Supplementary Loss Function Explanation and Analysis}

\subsection{Proof of Lemma~\ref{thm:opt}}
\label{app:proof:opt}
\begin{proof}
    Firstly, note that $\frac{\partial w_{ti}}{\partial z_j}=0,\forall i,j\in\{1,\dots,V\}$, by independence between the two. Also mind that $w_{ti}>0,\forall i\in\{1,\dots,V\}$.
    \begin{align}
    \frac{\partial\mathcal{L}_{\mathrm{oce}}(\mathbf{z},t)}{\partial z_t}
      &= -\frac{\sum_{i=1}^V w_{ti}e^{z_i}}{w_{tt}e^{z_t}}
      \notag\\
      &\cdot \frac{%
           w_{tt}e^{z_t}\sum_{i=1}^V w_{ti}e^{z_i}
           -w_{tt}e^{z_t}\cdot w_{tt}e^{z_t}
         }{%
           \bigl(\sum_{i=1}^V w_{ti}e^{z_i}\bigr)^2
         }
      \notag\\
      &= -\frac{\sum_{i=1}^V w_{ti}e^{z_i} - w_{tt}e^{z_t}}
               {\sum_{i=1}^V w_{ti}e^{z_i}}
      \\&< 0,
      \\[1ex]
    \frac{\partial\mathcal{L}_{\mathrm{oce}}(\mathbf{z},t)}{\partial z_i}
      &= -\frac{\sum_{j=1}^V w_{tj}e^{z_j}}{w_{tt}e^{z_t}}
         \cdot \frac{%
           -\,w_{ti}e^{z_i}\cdot w_{tt}e^{z_t}
         }{%
           \bigl(\sum_{j=1}^V w_{tj}e^{z_j}\bigr)^2
         }
      \notag\\
      &= \frac{w_{ti}e^{z_i}}
               {\sum_{j=1}^V w_{tj}e^{z_j}}> 0.
    \end{align}
    Thus, $\mathcal{L}_{\text{oce}}(\mathbf{z}, t)$ is optimized (minimized) when $z_t\to\infty,z_i\to-\infty (t\ne i)$.
\end{proof}

\subsection{Proof of Theorem~\ref{thm:adj}}
\label{app:prof:adj}
\begin{proof}
    Express $\mathcal{L}_{\text{oce}}$ and $\mathcal{L}_{\text{ce}}$ using $\mathbf{p}$, we have
    \begin{align}
  \mathcal{L}_{\text{ce}}(\mathbf{z},t)
  &= -\log p_t,\\[0.5ex]
  \mathcal{L}_{\text{oce}}(\mathbf{z},t)
  &= -\log\frac{w_{tt}e^{z_t}}{\sum_{i=1}^V w_{ti}e^{z_i}}
    \notag\\&= -\log\frac{w_{tt}\,\dfrac{e^{z_t}}{\sum_{i=1}^V e^{z_i}}}
               {\sum_{i=1}^V w_{ti}\,\dfrac{e^{z_i}}{\sum_{j=1}^V e^{z_j}}}
    \notag\\&= -\log\frac{w_{tt}p_t}{\sum_{i=1}^V w_{ti}p_i}
    \notag\\&= -\log\frac{w_{tt}p_t}{D(\mathbf{z})}.
\end{align}

    By $\mathcal{L}_{\text{ce}}(\mathbf{z},t)=\mathcal{L}_{\text{ce}}(\mathbf{z}^*,t)$, we have $p_t=p_t^*$. Given $D(\mathbf{z})<D(\mathbf{z}^*)$, we must have $\mathcal{L}_{\text{oce}}(\mathbf{z},t)<\mathcal{L}_{\text{oce}}(\mathbf{z}^*,t)$.
\end{proof}

\begin{figure}
    \centering
    \includegraphics[width=\linewidth,trim={15pt 18pt 5pt 17pt},clip]{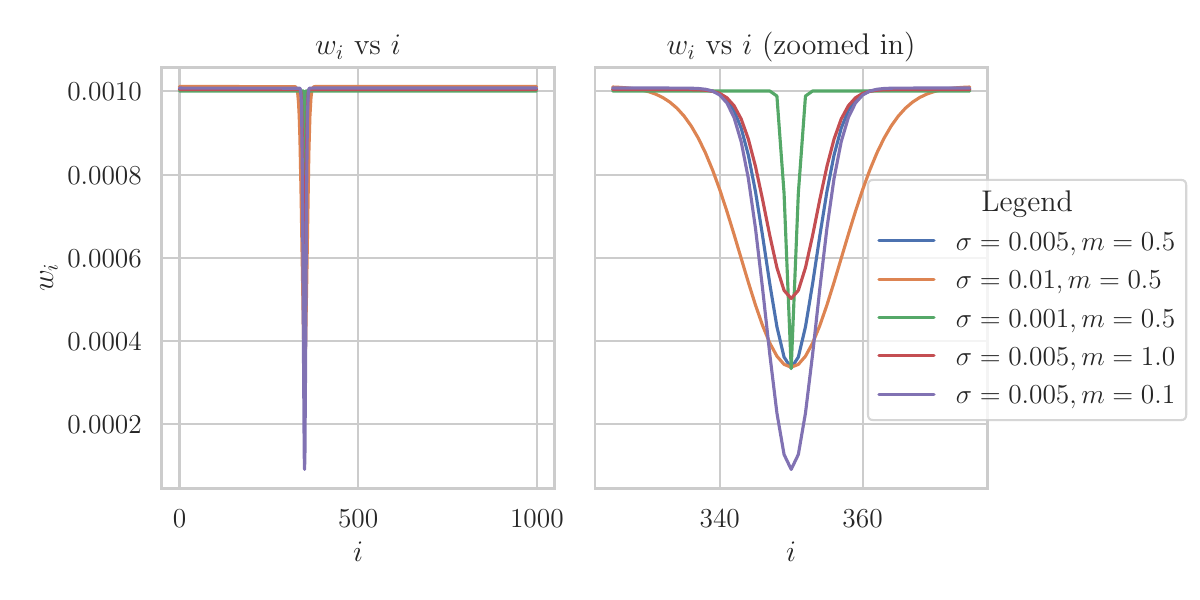}
    \caption{The relation of weights $w_{ti}$ (normalized so that $\|\mathbf{w}\|=1$ for visualization purpose to token index $i$, with vocabulary size $V=1000$, and target token $t=350$.}
    \label{fig:oce}
\end{figure}

\subsection{Detailed Explanations and Analysis of the Weight Function (Equation~\ref{eq:weight})}
\label{app:weight}

Recall the function of weight given the target token and current token in Equation~\ref{eq:weight}. This section will provide detailed explanations and analysis of it.

The exponential term comes from the intuition that the difference between closer quantiles is much more significant than the difference between farther quantiles. Some scaling and shifting is applied on the exponential term. 

Fig.~\ref{fig:oce} illustrates the relationship between $w_{ti}$ and $i$ under varying values of $\sigma$ and $m$. 
The parameter $\sigma$ primarily governs the width around the target token where the weights are significantly altered; larger $\sigma$ values result in a broader width. In contrast, $m$ dictates the intensity of the weight modification, with smaller $m$ values leading to more pronounced changes in weight.

\subsection{Overall Loss Computation}
\label{app:proof:overall}
Concerning the combination of valid quantile tokens and other non-quantile tokens, let valid quantile tokens have IDs in the range $[Q_s, Q_e)$.
The valid token group loss (for quantile tokens at this position) is defined by Equation~\ref{eq:vg} when $t \in [Q_s, Q_e)$.
\begin{equation}
\label{eq:vg}
    \mathcal{L}_{\text{vg}}(\mathbf{z},t)=-\log\frac{\sum_{i\in[Q_s,Q_e)}e^{z_i}}{\sum_{i=1}^Ve^{z_i}}
\end{equation}
In sum, Equation~\ref{eq:loss} expresses the loss function on one token.
\begin{align}
    \label{eq:loss}
    &\mathcal{L}(\mathbf{z},t)= \begin{cases}
        \mathcal{L}_{\text{ce}}(\mathbf{z},t)&\text{ if }t\notin[Q_s,Q_e)\\
        \mathcal{L}_{\text{oce}}(\mathbf{z}_{Q_s:Q_e},t-Q_s)+\mathcal{L}_{\text{vg}}(\mathbf{z},t)&\text{ if }t\in[Q_s,Q_e)\\
    \end{cases}
\end{align}

Although the ordinal property also holds for K-Means bin IDs, we do not calculate the loss on them this way because the number of different bins is small, and we expect the model to predict the bins accurately.


\section{Experiment Setup Details}
\label{app:setup}

\subsection{Experiment Environment}
\label{app:setup:env}

Experiments were conducted on a system running Ubuntu 22.04.1 LTS with kernel version 6.8.0. The machine featured a 13th Gen Intel Core i9-13900K CPU with 24 cores and 32 threads, capable of a maximum clock speed of 5.8 GHz. GPU computations were performed using a single NVIDIA GeForce RTX 4090 with 24 GB VRAM. The environment was set up with NVIDIA driver version 565.57.01 and CUDA 12.7.

\subsection{Details on Datasets}
\label{app:setup:datasets}

Table~\ref{tab:datasets} shows the details of the datasets we used. Datasets are downloaded by \verb|sklearn.datasets.fetch_openml|, with the dataset name as input.
\edit{Random splits of $8:2$ of training and test sets are applied.}

\begin{table*}[t]
    \centering
    \caption{Datasets used in experiments. \#R, \#F, \#N, and \#C represents the number of rows, features (including target column), numeric features, and categorical features respectively. The task can be classification, denoted ``clf'' followed by the number of classes in the bracket, and regression, denoted ``reg''. Aliases in bracket in ``Name'' column show the names used in this paper.}
    \label{tab:datasets}
    \setlength{\tabcolsep}{10pt}
    \resizebox{0.8\linewidth}{!}{
    \begin{tabular}{lrrrrl}
\toprule
Name & \#R & \#F & \#N & \#C & Task \\
\midrule
adult~\cite{adult-data} & 48842 & 15 & 2 & 13 & clf(2) \\
bank-marketing (bank)~\cite{bank-data} & 45211 & 17 & 7 & 10 & clf(2) \\
boston~\cite{boston-data} & 506 & 14 & 12 & 2 & reg \\
breast-w (breast)~\cite{breast-data} & 699 & 10 & 9 & 1 & clf(2) \\
credit-g (credit)~\cite{credit-data} & 1000 & 21 & 7 & 14 & clf(2) \\
diabetes~\cite{diabetes-data} & 768 & 9 & 8 & 1 & clf(2) \\
iris~\cite{iris-data} & 150 & 5 & 4 & 1 & clf(3) \\
qsar-biodeg (qsar)~\cite{qsar-data} & 1055 & 42 & 41 & 1 & clf(2) \\
wdbc~\cite{wdbc-data} & 569 & 31 & 30 & 1 & clf(2) \\
\bottomrule
\end{tabular}

    }
\end{table*}

\begin{table*}[t]
    \caption{Performance comparison between different models in terms of MLE. Classification tasks are evaluated by weighted AUC ROC scores, and regression tasks are evaluated by $R^2$ scores. The last row shows the average relative error versus real. The best scores and the second best scores are highlighted in bold with and without underscore, respectively. Equal values in the first 3 digits may be compared by the 4-th or 5-th digits.}
    \label{tab:ml-utility}
    \setlength{\tabcolsep}{1pt}
    \resizebox{\linewidth}{!}{

\begin{tabular}{ll>{\columncolor{gray!20}}ccccccccc>{\columncolor{red!10}}c>{\columncolor{red!10}}c>{\columncolor{red!10}}c}
\toprule
Dataset & ML & Real & CTAB+ & TTVAE & TabSyn & FD & GReaT & RTF & TabMT & TabuLa & TTF-S & TTF-L & TTF-NP \\
\midrule
\multirow{3}{*}{adult} & LN & 0.914 & $0.904_{\pm 0.001}$ & $0.875_{\pm 0.010}$ & $0.910_{\pm 0.002}$ & $\boldsymbol{0.911}_{\pm 0.000}$ & $0.911_{\pm 0.000}$ & $0.911_{\pm 0.000}$ & $0.911_{\pm 0.000}$ & $\underline{\boldsymbol{0.913}}_{\pm 0.000}$ & $0.906_{\pm 0.001}$ & $0.911_{\pm 0.001}$ & $0.910_{\pm 0.000}$ \\
 & RF & 0.910 & $0.902_{\pm 0.002}$ & $0.869_{\pm 0.011}$ & $0.905_{\pm 0.003}$ & $\boldsymbol{0.908}_{\pm 0.000}$ & $0.905_{\pm 0.000}$ & $\underline{\boldsymbol{0.908}}_{\pm 0.000}$ & $0.907_{\pm 0.002}$ & $0.888_{\pm 0.002}$ & $0.901_{\pm 0.001}$ & $0.905_{\pm 0.000}$ & $0.903_{\pm 0.002}$ \\
 & XGB & 0.914 & $0.905_{\pm 0.001}$ & $0.870_{\pm 0.012}$ & $0.910_{\pm 0.003}$ & $\boldsymbol{0.912}_{\pm 0.001}$ & $0.911_{\pm 0.000}$ & $\underline{\boldsymbol{0.912}}_{\pm 0.001}$ & $0.911_{\pm 0.001}$ & $0.899_{\pm 0.005}$ & $0.906_{\pm 0.001}$ & $0.911_{\pm 0.001}$ & $0.909_{\pm 0.001}$ \\
\midrule
\multirow{3}{*}{bank} & LN & 0.907 & $0.901_{\pm 0.003}$ & $0.885_{\pm 0.004}$ & $0.900_{\pm 0.005}$ & $0.903_{\pm 0.001}$ & $\underline{\boldsymbol{0.907}}_{\pm 0.000}$ & $0.904_{\pm 0.000}$ & $0.862_{\pm 0.005}$ & $0.906_{\pm 0.001}$ & $0.906_{\pm 0.000}$ & $\boldsymbol{0.906}_{\pm 0.000}$ & $0.901_{\pm 0.003}$ \\
 & RF & 0.930 & $0.902_{\pm 0.002}$ & $0.884_{\pm 0.009}$ & $0.907_{\pm 0.006}$ & $0.917_{\pm 0.001}$ & $0.907_{\pm 0.001}$ & $0.913_{\pm 0.001}$ & $0.886_{\pm 0.011}$ & $\underline{\boldsymbol{0.923}}_{\pm 0.004}$ & $0.913_{\pm 0.002}$ & $0.918_{\pm 0.001}$ & $\boldsymbol{0.921}_{\pm 0.003}$ \\
 & XGB & 0.936 & $0.906_{\pm 0.000}$ & $0.877_{\pm 0.011}$ & $0.913_{\pm 0.007}$ & $0.920_{\pm 0.001}$ & $0.910_{\pm 0.001}$ & $0.918_{\pm 0.002}$ & $0.892_{\pm 0.007}$ & $0.925_{\pm 0.003}$ & $0.919_{\pm 0.002}$ & $\boldsymbol{0.925}_{\pm 0.003}$ & $\underline{\boldsymbol{0.926}}_{\pm 0.002}$ \\
\midrule
\multirow{3}{*}{boston} & LN & 0.590 & $0.264_{\pm 0.034}$ & $0.000_{\pm 0.079}$ & $0.533_{\pm 0.036}$ & $0.542_{\pm 0.016}$ & $0.397_{\pm 0.000}$ & $0.529_{\pm 0.028}$ & $0.409_{\pm 0.015}$ & $0.484_{\pm 0.069}$ & $0.525_{\pm 0.035}$ & $\boldsymbol{0.575}_{\pm 0.011}$ & $\underline{\boldsymbol{0.591}}_{\pm 0.000}$ \\
 & RF & 0.662 & $0.129_{\pm 0.015}$ & $0.124_{\pm 0.068}$ & $0.658_{\pm 0.025}$ & $0.608_{\pm 0.008}$ & $0.281_{\pm 0.001}$ & $0.584_{\pm 0.006}$ & $0.439_{\pm 0.096}$ & $0.589_{\pm 0.050}$ & $0.597_{\pm 0.038}$ & $\underline{\boldsymbol{0.694}}_{\pm 0.041}$ & $\boldsymbol{0.683}_{\pm 0.016}$ \\
 & XGB & 0.701 & $0.097_{\pm 0.078}$ & $0.211_{\pm 0.121}$ & $0.647_{\pm 0.017}$ & $0.604_{\pm 0.024}$ & $0.277_{\pm 0.014}$ & $0.578_{\pm 0.056}$ & $0.446_{\pm 0.115}$ & $0.615_{\pm 0.069}$ & $0.594_{\pm 0.061}$ & $\boldsymbol{0.665}_{\pm 0.026}$ & $\underline{\boldsymbol{0.691}}_{\pm 0.023}$ \\
\midrule
\multirow{3}{*}{breast} & LN & 0.985 & $0.913_{\pm 0.091}$ & $0.963_{\pm 0.006}$ & $\boldsymbol{0.987}_{\pm 0.001}$ & $0.986_{\pm 0.002}$ & $0.985_{\pm 0.000}$ & $\underline{\boldsymbol{0.987}}_{\pm 0.002}$ & $0.983_{\pm 0.004}$ & $0.985_{\pm 0.001}$ & $0.981_{\pm 0.005}$ & $0.987_{\pm 0.002}$ & $0.984_{\pm 0.004}$ \\
 & RF & 0.985 & $0.968_{\pm 0.012}$ & $0.979_{\pm 0.004}$ & $0.981_{\pm 0.002}$ & $0.975_{\pm 0.005}$ & $0.979_{\pm 0.004}$ & $0.980_{\pm 0.000}$ & $0.978_{\pm 0.002}$ & $0.976_{\pm 0.005}$ & $0.978_{\pm 0.001}$ & $\boldsymbol{0.983}_{\pm 0.002}$ & $\underline{\boldsymbol{0.984}}_{\pm 0.002}$ \\
 & XGB & 0.984 & $0.958_{\pm 0.025}$ & $0.970_{\pm 0.009}$ & $\underline{\boldsymbol{0.987}}_{\pm 0.001}$ & $\boldsymbol{0.984}_{\pm 0.002}$ & $0.982_{\pm 0.002}$ & $0.982_{\pm 0.002}$ & $0.982_{\pm 0.002}$ & $0.979_{\pm 0.002}$ & $0.982_{\pm 0.002}$ & $0.983_{\pm 0.003}$ & $0.982_{\pm 0.003}$ \\
\midrule
\multirow{3}{*}{credit} & LN & 0.836 & $0.538_{\pm 0.093}$ & $0.500_{\pm 0.000}$ & $0.780_{\pm 0.017}$ & $\boldsymbol{0.827}_{\pm 0.010}$ & $0.665_{\pm 0.000}$ & $0.801_{\pm 0.006}$ & $0.773_{\pm 0.038}$ & $0.814_{\pm 0.016}$ & $0.816_{\pm 0.020}$ & $0.755_{\pm 0.011}$ & $\underline{\boldsymbol{0.834}}_{\pm 0.001}$ \\
 & RF & 0.837 & $0.505_{\pm 0.097}$ & $0.500_{\pm 0.000}$ & $0.791_{\pm 0.007}$ & $\boldsymbol{0.819}_{\pm 0.009}$ & $0.722_{\pm 0.014}$ & $\underline{\boldsymbol{0.820}}_{\pm 0.010}$ & $0.763_{\pm 0.019}$ & $0.805_{\pm 0.019}$ & $0.781_{\pm 0.026}$ & $0.741_{\pm 0.027}$ & $0.813_{\pm 0.026}$ \\
 & XGB & 0.844 & $0.509_{\pm 0.088}$ & $0.500_{\pm 0.000}$ & $0.795_{\pm 0.014}$ & $\underline{\boldsymbol{0.836}}_{\pm 0.010}$ & $0.739_{\pm 0.006}$ & $0.797_{\pm 0.039}$ & $0.761_{\pm 0.030}$ & $0.817_{\pm 0.020}$ & $0.796_{\pm 0.011}$ & $0.764_{\pm 0.014}$ & $\boldsymbol{0.826}_{\pm 0.026}$ \\
\midrule
\multirow{3}{*}{diabetes} & LN & 0.884 & $0.817_{\pm 0.028}$ & $0.880_{\pm 0.003}$ & $0.883_{\pm 0.006}$ & $0.885_{\pm 0.001}$ & $0.858_{\pm 0.000}$ & $0.878_{\pm 0.003}$ & $0.855_{\pm 0.014}$ & $0.876_{\pm 0.007}$ & $\underline{\boldsymbol{0.890}}_{\pm 0.010}$ & $\boldsymbol{0.886}_{\pm 0.001}$ & $0.843_{\pm 0.058}$ \\
 & RF & 0.869 & $0.802_{\pm 0.014}$ & $0.849_{\pm 0.004}$ & $0.844_{\pm 0.012}$ & $0.856_{\pm 0.012}$ & $0.824_{\pm 0.012}$ & $0.841_{\pm 0.010}$ & $0.845_{\pm 0.009}$ & $\underline{\boldsymbol{0.863}}_{\pm 0.003}$ & $0.855_{\pm 0.014}$ & $\boldsymbol{0.859}_{\pm 0.010}$ & $0.814_{\pm 0.073}$ \\
 & XGB & 0.867 & $0.811_{\pm 0.005}$ & $0.851_{\pm 0.001}$ & $0.847_{\pm 0.009}$ & $0.846_{\pm 0.012}$ & $0.831_{\pm 0.003}$ & $0.849_{\pm 0.010}$ & $0.840_{\pm 0.009}$ & $\underline{\boldsymbol{0.857}}_{\pm 0.009}$ & $\boldsymbol{0.852}_{\pm 0.015}$ & $0.848_{\pm 0.017}$ & $0.811_{\pm 0.085}$ \\
\midrule
\multirow{3}{*}{iris} & LN & 1.000 & $0.272_{\pm 0.068}$ & $0.996_{\pm 0.006}$ & $0.999_{\pm 0.002}$ & $0.997_{\pm 0.005}$ & $0.985_{\pm 0.000}$ & $0.983_{\pm 0.015}$ & $\underline{\boldsymbol{1.000}}_{\pm 0.000}$ &  & $0.971_{\pm 0.018}$ & $0.997_{\pm 0.005}$ & $\underline{\boldsymbol{1.000}}_{\pm 0.000}$ \\
 & RF & 1.000 & $0.396_{\pm 0.099}$ & $0.992_{\pm 0.012}$ & $\underline{\boldsymbol{1.000}}_{\pm 0.000}$ & $\underline{\boldsymbol{1.000}}_{\pm 0.000}$ & $\underline{\boldsymbol{1.000}}_{\pm 0.000}$ & $0.993_{\pm 0.005}$ & $0.993_{\pm 0.005}$ & - & $0.999_{\pm 0.002}$ & $0.988_{\pm 0.017}$ & $0.967_{\pm 0.047}$ \\
 & XGB & 1.000 & $0.211_{\pm 0.042}$ & $0.999_{\pm 0.002}$ & $\underline{\boldsymbol{1.000}}_{\pm 0.000}$ & $\underline{\boldsymbol{1.000}}_{\pm 0.000}$ & $\underline{\boldsymbol{1.000}}_{\pm 0.000}$ & $0.998_{\pm 0.003}$ & $\underline{\boldsymbol{1.000}}_{\pm 0.000}$ &  & $\underline{\boldsymbol{1.000}}_{\pm 0.000}$ & $\underline{\boldsymbol{1.000}}_{\pm 0.000}$ & $0.999_{\pm 0.001}$ \\
\midrule
\multirow{3}{*}{qsar} & LN & 0.906 & $0.716_{\pm 0.085}$ & $0.669_{\pm 0.120}$ & $0.867_{\pm 0.012}$ & $0.881_{\pm 0.005}$ & $0.673_{\pm 0.000}$ & $0.880_{\pm 0.003}$ & $\boldsymbol{0.887}_{\pm 0.005}$ & $0.868_{\pm 0.034}$ & $0.874_{\pm 0.013}$ & $0.876_{\pm 0.006}$ & $\underline{\boldsymbol{0.907}}_{\pm 0.000}$ \\
 & RF & 0.936 & $0.648_{\pm 0.036}$ & $0.666_{\pm 0.117}$ & $0.882_{\pm 0.006}$ & $0.897_{\pm 0.011}$ & $0.616_{\pm 0.009}$ & $\boldsymbol{0.907}_{\pm 0.007}$ & $0.880_{\pm 0.008}$ & $0.825_{\pm 0.028}$ & $0.873_{\pm 0.006}$ & $0.884_{\pm 0.004}$ & $\underline{\boldsymbol{0.925}}_{\pm 0.006}$ \\
 & XGB & 0.921 & $0.731_{\pm 0.020}$ & $0.658_{\pm 0.114}$ & $0.860_{\pm 0.016}$ & $0.885_{\pm 0.010}$ & $0.602_{\pm 0.020}$ & $\boldsymbol{0.897}_{\pm 0.010}$ & $0.873_{\pm 0.009}$ & $0.842_{\pm 0.023}$ & $0.862_{\pm 0.005}$ & $0.884_{\pm 0.004}$ & $\underline{\boldsymbol{0.917}}_{\pm 0.005}$ \\
\midrule
\multirow{3}{*}{wdbc} & LN & 0.993 & $0.929_{\pm 0.053}$ & $0.976_{\pm 0.019}$ & $0.992_{\pm 0.001}$ & $\boldsymbol{0.993}_{\pm 0.002}$ & $0.979_{\pm 0.000}$ & $0.988_{\pm 0.003}$ & $0.984_{\pm 0.006}$ &  & $0.980_{\pm 0.007}$ & $0.979_{\pm 0.007}$ & $\underline{\boldsymbol{0.993}}_{\pm 0.000}$ \\
 & RF & 0.976 & $0.919_{\pm 0.024}$ & $0.952_{\pm 0.041}$ & $\underline{\boldsymbol{0.984}}_{\pm 0.004}$ & $\boldsymbol{0.982}_{\pm 0.004}$ & $0.976_{\pm 0.003}$ & $0.980_{\pm 0.012}$ & $0.978_{\pm 0.003}$ & - & $0.981_{\pm 0.001}$ & $0.975_{\pm 0.002}$ & $0.981_{\pm 0.001}$ \\
 & XGB & 0.990 & $0.930_{\pm 0.021}$ & $0.949_{\pm 0.035}$ & $0.986_{\pm 0.002}$ & $\boldsymbol{0.989}_{\pm 0.001}$ & $0.973_{\pm 0.003}$ & $0.988_{\pm 0.002}$ & $0.982_{\pm 0.004}$ &  & $0.987_{\pm 0.003}$ & $0.985_{\pm 0.006}$ & $\underline{\boldsymbol{0.989}}_{\pm 0.001}$ \\
\bottomrule
\end{tabular}

    }
\end{table*}

\subsection{Baseline Choices and Implementation}
\label{app:setup:baseline}

For i) non-neural-network method (tree-based sampling), ii) GAN~\citep{gan}, iii) VAE~\citep{vae}, and iv) Diffusion~\citep{diffusion}, we choose one recent, representative, and well-performing (especially on MLE) model each: Forest Diffusion~\citep{fdiff}, CTAB-GAN+~\citep{ctabganp}, TTVAE~\citep{ttvae}, and TabSyn~\citep{tabsyn}.
For auto-regressive and masked transformers, which are more aligned with our work, we include more works: GReaT~\citep{great}, REaLTabFormer~\citep{realtabformer}, TabMT~\citep{tabmt}, TabuLa~\citep{tabula}.
All of them will use a DistilGPT2~\citep{distilbert} as the transformer backbone except for TabMT~\citep{tabmt}, which has specific settings described in its paper.

Other large language model (LLM)-based methods, such as HARMONIC~\citep{harmonic}, are excluded from our comparison for fairer comparison because 
\begin{itemize}
    \item They usually require a good-quality dataset description, but we want to focus on the case without additional information besides the data per se.
    \item Their performance is dependent on a very large and good LLM. The improvement of the LLM model may be more important than the improvement of the tabular-specific design, but we want to showcase the effectiveness of tabular-specific designs instead of an LLM.
\end{itemize}

We use the Synthcity~\citep{synthcity} implementation for 
GReaT~\citep{great}.
Our benchmark code is adapted from Synthcity~\citep{synthcity} too.
We also include CTAB-GAN+~\citep{ctabganp}, REaLTabFormer~\citep{realtabformer}, TabSyn~\citep{tabsyn}, Forest Diffusion~\citep{fdiff}, TabuLa~\citep{tabula}, and TTVAE~\citep{ttvae} using their official code on GitHub or public SDK in the benchmark.
TabMT~\citep{tabmt} has no publicly available code, so we do experiments based on a reproduction of the model based on the paper's description, so the model could be slightly different from that paper.

\subsection{Machine Learning Efficacy (MLE) Metrics Implementation}
\label{app:setup:ml-metric}

Instead of running all models with fixed, typically default, parameters across all datasets, we perform hyperparameter tuning before reporting their performance. The tuning process follows the same methodology used to optimize the core tree-based model in \frameworkname, but with 30 trials per round of optimization.

For linear models, no hyperparameter tuning is performed, as they follow the classical definition without additional tunable parameters.

The hyperparameter space explored for random forest is
\begin{itemize}
    \item \textbf{Number of Estimators}: Integers in $[100,300]$ step 50;
    \item \textbf{Maximum Depth}: Integers in $[5, 20]$ step 5;
    \item \textbf{Minimum Samples Split}: Integers in $[2,10]$ step 2;
    \item \textbf{Minimum Size per Leaf}: Integers $[1,5]$;
    \item \textbf{Maximum Features}: Value in ``sqrt'', ``log2'', and NULL;
    \item \textbf{Bootstramp}: Enabled or disabled.
\end{itemize}

The hyperparameter space explored for XGBoost is:
\begin{itemize}
    \item \textbf{Learning Rate}: Logarithmic scale float in $[0.01,0.3]$;
    \item \textbf{Number of Estimators}: Integers in $[100,300]$ step 50;
    \item \textbf{Maximum Depth}: Integers in $[3,10]$;
    \item \textbf{Minimum Child Weight}: Logarithmic scale float in $[1.0,10.0]$;
    \item \textbf{Minimum Split Loss Gamma}: Float in $[0.0,0.5]$;
    \item \textbf{Subsample Ratio}: Float in $[0.5,1.0]$;
    \item \textbf{Subsample Ratio of Columns by Tree}: Float in $[0.5,1.0]$;
    \item \textbf{L1 Regularization}: Float in $[0.0,10.0]$.
\end{itemize}

For all models, numeric values are standardized using standard scaling. Categorical values are one-hot encoded for linear models and label-encoded for random forest and XGBoost. To simplify the experiments, rows with missing values are excluded, as handling missing data is not the focus of this paper.

The reported performance for machine learning utility is measured using the weighted AUC-ROC for classification tasks and the $R^2$ score for regression tasks, so for all these scores, a larger value indicates a better performance. Each generator is trained and used to generate synthetic data of the same size as the training dataset three times for each dataset, and the summarized results of these runs are reported.

\begin{table*}[t]
    \centering
    \caption{Performance comparison between different models in terms of fidelity. The closer the values are to 1, the better the fidelity is. The best scores and the second best scores are highlighted in bold with and without underscore, respectively. Equal values in the first 3 digits may be compared by the 4-th or 5-th digits.}
    \label{tab:fidelity}
    \setlength{\tabcolsep}{1pt}
    \resizebox{\linewidth}{!}{
    \begin{tabular}{llcccccccc>{\columncolor{red!10}}c>{\columncolor{red!10}}c>{\columncolor{red!10}}c>{\columncolor{red!10}}c>{\columncolor{red!10}}c>{\columncolor{red!10}}c}
\toprule
Dataset & Metric & CTAB+ & TTVAE & TabSyn & FD & GReaT & RTF & TabMT & TabuLa & TTF-S & TTF-L & TTF-NP \\
\midrule
\multirow{2}{*}{adult} & Shape & $0.964_{\pm 0.005}$ & $0.852_{\pm 0.024}$ & $\boldsymbol{0.977}_{\pm 0.011}$ & $\underline{\boldsymbol{0.986}}_{\pm 0.000}$ & $0.929_{\pm 0.000}$ & $0.962_{\pm 0.001}$ & $0.974_{\pm 0.003}$ & $0.975_{\pm 0.000}$ & $0.951_{\pm 0.002}$ & $0.957_{\pm 0.003}$ & $0.921_{\pm 0.003}$ \\
 & Trend & $0.915_{\pm 0.004}$ & $0.733_{\pm 0.032}$ & $0.955_{\pm 0.018}$ & $\underline{\boldsymbol{0.972}}_{\pm 0.000}$ & $0.882_{\pm 0.000}$ & $0.933_{\pm 0.003}$ & $\boldsymbol{0.957}_{\pm 0.004}$ & $0.956_{\pm 0.001}$ & $0.887_{\pm 0.003}$ & $0.895_{\pm 0.005}$ & $0.838_{\pm 0.027}$ \\
\midrule
\multirow{2}{*}{bank} & Shape & $0.965_{\pm 0.002}$ & $0.875_{\pm 0.009}$ & $\boldsymbol{0.983}_{\pm 0.005}$ & $0.970_{\pm 0.001}$ & $0.915_{\pm 0.000}$ & $0.967_{\pm 0.001}$ & $0.969_{\pm 0.009}$ & $\underline{\boldsymbol{0.983}}_{\pm 0.000}$ & $0.927_{\pm 0.003}$ & $0.935_{\pm 0.003}$ & $0.932_{\pm 0.001}$ \\
 & Trend & $0.884_{\pm 0.003}$ & $0.788_{\pm 0.017}$ & $\boldsymbol{0.967}_{\pm 0.008}$ & $0.965_{\pm 0.019}$ & $0.903_{\pm 0.000}$ & $0.955_{\pm 0.001}$ & $0.944_{\pm 0.034}$ & $\underline{\boldsymbol{0.971}}_{\pm 0.002}$ & $0.944_{\pm 0.006}$ & $0.956_{\pm 0.003}$ & $0.935_{\pm 0.027}$ \\
\midrule
\multirow{2}{*}{boston} & Shape & $0.850_{\pm 0.003}$ & $0.663_{\pm 0.023}$ & $0.893_{\pm 0.003}$ & $0.895_{\pm 0.002}$ & $\boldsymbol{0.916}_{\pm 0.000}$ & $\underline{\boldsymbol{0.942}}_{\pm 0.004}$ & $0.835_{\pm 0.009}$ & $0.861_{\pm 0.030}$ & $0.843_{\pm 0.006}$ & $0.855_{\pm 0.003}$ & $0.863_{\pm 0.000}$ \\
 & Trend & $0.828_{\pm 0.011}$ & $0.747_{\pm 0.021}$ & $0.934_{\pm 0.003}$ & $0.942_{\pm 0.001}$ & $0.908_{\pm 0.000}$ & $0.947_{\pm 0.002}$ & $0.935_{\pm 0.014}$ & $0.886_{\pm 0.004}$ & $0.951_{\pm 0.020}$ & $\boldsymbol{0.973}_{\pm 0.006}$ & $\underline{\boldsymbol{0.992}}_{\pm 0.000}$ \\
\midrule
\multirow{2}{*}{breast} & Shape & $0.775_{\pm 0.024}$ & $0.498_{\pm 0.052}$ & $0.813_{\pm 0.005}$ & $0.835_{\pm 0.006}$ & $0.728_{\pm 0.000}$ & $0.753_{\pm 0.014}$ & $0.841_{\pm 0.048}$ & $0.814_{\pm 0.003}$ & $\underline{\boldsymbol{0.866}}_{\pm 0.022}$ & $\boldsymbol{0.848}_{\pm 0.010}$ & $0.758_{\pm 0.116}$ \\
 & Trend & $0.616_{\pm 0.063}$ & $0.333_{\pm 0.065}$ & $0.709_{\pm 0.006}$ & $\boldsymbol{0.766}_{\pm 0.004}$ & $0.636_{\pm 0.000}$ & $0.668_{\pm 0.010}$ & $0.747_{\pm 0.035}$ & $\underline{\boldsymbol{0.767}}_{\pm 0.003}$ & $0.735_{\pm 0.025}$ & $0.728_{\pm 0.020}$ & $0.680_{\pm 0.165}$ \\
\midrule
\multirow[t]{2}{*}{credit} & Shape & $0.937_{\pm 0.013}$ & $0.555_{\pm 0.001}$ & $0.932_{\pm 0.007}$ & $0.960_{\pm 0.001}$ & $0.932_{\pm 0.000}$ & $0.944_{\pm 0.003}$ & $0.955_{\pm 0.004}$ & $\boldsymbol{0.970}_{\pm 0.002}$ & $0.941_{\pm 0.008}$ & $0.946_{\pm 0.010}$ & $\underline{\boldsymbol{0.992}}_{\pm 0.008}$ \\
 & Trend & $0.862_{\pm 0.030}$ & $0.306_{\pm 0.005}$ & $0.862_{\pm 0.008}$ & $0.912_{\pm 0.009}$ & $0.861_{\pm 0.000}$ & $0.896_{\pm 0.006}$ & $0.911_{\pm 0.004}$ & $\boldsymbol{0.917}_{\pm 0.001}$ & $0.880_{\pm 0.012}$ & $0.883_{\pm 0.018}$ & $\underline{\boldsymbol{0.958}}_{\pm 0.019}$ \\
\midrule
\multirow[t]{2}{*}{diabetes} & Shape & $0.923_{\pm 0.006}$ & $0.887_{\pm 0.017}$ & $\boldsymbol{0.955}_{\pm 0.007}$ & $0.943_{\pm 0.003}$ & $0.881_{\pm 0.000}$ & $0.943_{\pm 0.004}$ & $0.946_{\pm 0.004}$ & $\underline{\boldsymbol{0.975}}_{\pm 0.002}$ & $0.928_{\pm 0.002}$ & $0.919_{\pm 0.014}$ & $0.806_{\pm 0.212}$ \\
 & Trend & $0.898_{\pm 0.004}$ & $0.888_{\pm 0.012}$ & $\boldsymbol{0.964}_{\pm 0.002}$ & $0.957_{\pm 0.006}$ & $0.921_{\pm 0.000}$ & $0.952_{\pm 0.005}$ & $0.952_{\pm 0.003}$ & $\underline{\boldsymbol{0.971}}_{\pm 0.012}$ & $0.954_{\pm 0.004}$ & $0.957_{\pm 0.000}$ & $0.924_{\pm 0.091}$ \\
\midrule
\multirow{2}{*}{iris} & Shape & $0.822_{\pm 0.022}$ & $0.885_{\pm 0.008}$ & $0.893_{\pm 0.003}$ & $\boldsymbol{0.901}_{\pm 0.004}$ & $0.870_{\pm 0.000}$ & $\underline{\boldsymbol{0.906}}_{\pm 0.017}$ & $0.895_{\pm 0.007}$ & \multirow{2}{*}{-} & $0.886_{\pm 0.011}$ & $0.883_{\pm 0.019}$ & $0.839_{\pm 0.029}$ \\
 & Trend & $0.590_{\pm 0.019}$ & $0.878_{\pm 0.004}$ & $0.917_{\pm 0.007}$ & $\boldsymbol{0.920}_{\pm 0.006}$ & $0.899_{\pm 0.000}$ & $0.896_{\pm 0.016}$ & $0.901_{\pm 0.011}$ &  & $0.902_{\pm 0.005}$ & $\underline{\boldsymbol{0.922}}_{\pm 0.018}$ & $0.895_{\pm 0.055}$ \\
\midrule
\multirow{2}{*}{qsar} & Shape & $0.917_{\pm 0.003}$ & $0.652_{\pm 0.031}$ & $0.930_{\pm 0.005}$ & $0.947_{\pm 0.004}$ & $0.887_{\pm 0.000}$ & $\underline{\boldsymbol{0.959}}_{\pm 0.003}$ & $0.920_{\pm 0.006}$ & $0.833_{\pm 0.034}$ & $0.914_{\pm 0.004}$ & $0.936_{\pm 0.003}$ & $\boldsymbol{0.954}_{\pm 0.001}$ \\
 & Trend & $0.863_{\pm 0.007}$ & $0.610_{\pm 0.059}$ & $0.914_{\pm 0.009}$ & $0.898_{\pm 0.011}$ & $0.855_{\pm 0.000}$ & $0.890_{\pm 0.007}$ & $0.923_{\pm 0.004}$ & $0.792_{\pm 0.008}$ & $0.896_{\pm 0.010}$ & $\boldsymbol{0.935}_{\pm 0.008}$ & $\underline{\boldsymbol{0.956}}_{\pm 0.001}$ \\
\midrule
\multirow{2}{*}{wdbc} & Shape & $0.861_{\pm 0.005}$ & $0.754_{\pm 0.047}$ & $0.946_{\pm 0.004}$ & $0.946_{\pm 0.004}$ & $0.873_{\pm 0.000}$ & $0.945_{\pm 0.006}$ & $0.937_{\pm 0.008}$ & \multirow{2}{*}{-} & $0.938_{\pm 0.007}$ & $\boldsymbol{0.957}_{\pm 0.003}$ & $\underline{\boldsymbol{0.980}}_{\pm 0.004}$ \\
 & Trend & $0.866_{\pm 0.007}$ & $0.817_{\pm 0.049}$ & $\boldsymbol{0.976}_{\pm 0.002}$ & $0.962_{\pm 0.001}$ & $0.896_{\pm 0.000}$ & $0.954_{\pm 0.001}$ & $0.943_{\pm 0.001}$ &  & $0.960_{\pm 0.002}$ & $0.966_{\pm 0.006}$ & $\underline{\boldsymbol{0.994}}_{\pm 0.002}$ \\
\bottomrule
\end{tabular}

    }
\end{table*}

\subsection{Fidelity Metrics Implementation}
\label{app:setup:fidelity}

\begin{table*}[t]
    \caption{Performance comparison between different models in terms of DCR. $p$-values less than $0.05$ is highlighted in red, indicating a high risk of privacy leakage.}
    \label{tab:dcr}
    \setlength{\tabcolsep}{1pt}
    \resizebox{\linewidth}{!}{
    \begin{tabular}{l|cccccccc>{\columncolor{red!10}}c>{\columncolor{red!10}}c>{\columncolor{red!10}}c>{\columncolor{red!10}}c>{\columncolor{red!10}}c>{\columncolor{red!10}}c}
\toprule
Dataset & CTAB+ & TTVAE & TabSyn & FD & GReaT & RTF & TabMT & TabuLa & TTF-S & TTF-L & TTF-NP \\
\midrule
adult & $1.000_{\pm 0.000}$ & $1.000_{\pm 0.000}$ & $1.000_{\pm 0.000}$ & $1.000_{\pm 0.000}$ & $0.857_{\pm 0.000}$ & \textcolor{red}{$0.000_{\pm 0.000}$} & $0.998_{\pm 0.002}$ & \textcolor{red}{$0.000_{\pm 0.000}$} & $1.000_{\pm 0.000}$ & $1.000_{\pm 0.000}$ & $1.000_{\pm 0.000}$ \\
bank & $1.000_{\pm 0.000}$ & $1.000_{\pm 0.000}$ & $1.000_{\pm 0.000}$ & $0.394_{\pm 0.082}$ & $1.000_{\pm 0.000}$ & $1.000_{\pm 0.000}$ & $1.000_{\pm 0.000}$ & \textcolor{red}{$0.000_{\pm 0.000}$} & $1.000_{\pm 0.000}$ & $1.000_{\pm 0.000}$ & $1.000_{\pm 0.000}$ \\
boston & $1.000_{\pm 0.000}$ & $1.000_{\pm 0.000}$ & $0.941_{\pm 0.040}$ & $0.996_{\pm 0.005}$ & $1.000_{\pm 0.000}$ & $0.508_{\pm 0.295}$ & $1.000_{\pm 0.000}$ & $0.754_{\pm 0.121}$ & $1.000_{\pm 0.000}$ & $1.000_{\pm 0.000}$ & $1.000_{\pm 0.000}$ \\
breast & $1.000_{\pm 0.000}$ & $1.000_{\pm 0.000}$ & $0.231_{\pm 0.105}$ & \textcolor{red}{$0.003_{\pm 0.001}$} & $0.082_{\pm 0.000}$ & \textcolor{red}{$0.002_{\pm 0.001}$} & $0.259_{\pm 0.064}$ & \textcolor{red}{$0.000_{\pm 0.000}$} & $0.990_{\pm 0.012}$ & $0.996_{\pm 0.005}$ & $0.332_{\pm 0.470}$ \\
credit & $1.000_{\pm 0.000}$ & $1.000_{\pm 0.000}$ & $0.186_{\pm 0.141}$ & \textcolor{red}{$0.005_{\pm 0.003}$} & $0.883_{\pm 0.000}$ & \textcolor{red}{$0.000_{\pm 0.000}$} & $0.824_{\pm 0.183}$ & \textcolor{red}{$0.000_{\pm 0.000}$} & $0.908_{\pm 0.107}$ & $0.676_{\pm 0.455}$ & \textcolor{red}{$0.000_{\pm 0.000}$} \\
diabetes & $1.000_{\pm 0.000}$ & $0.753_{\pm 0.288}$ & \textcolor{red}{$0.029_{\pm 0.024}$} & \textcolor{red}{$0.000_{\pm 0.000}$} & $0.998_{\pm 0.000}$ & $0.077_{\pm 0.076}$ & $0.876_{\pm 0.149}$ & \textcolor{red}{$0.000_{\pm 0.000}$} & $0.867_{\pm 0.158}$ & $0.603_{\pm 0.310}$ & $0.333_{\pm 0.471}$ \\
iris & $1.000_{\pm 0.000}$ & $0.382_{\pm 0.229}$ & $0.552_{\pm 0.099}$ & \textcolor{red}{$0.049_{\pm 0.032}$} & $0.764_{\pm 0.000}$ & $0.638_{\pm 0.226}$ & $0.824_{\pm 0.191}$ & - & $0.662_{\pm 0.421}$ & $0.214_{\pm 0.022}$ & $0.086_{\pm 0.096}$ \\
qsar & $1.000_{\pm 0.000}$ & $1.000_{\pm 0.000}$ & $1.000_{\pm 0.000}$ & $0.977_{\pm 0.019}$ & $1.000_{\pm 0.000}$ & $0.326_{\pm 0.088}$ & $1.000_{\pm 0.000}$ & $1.000_{\pm 0.000}$ & $1.000_{\pm 0.000}$ & $0.956_{\pm 0.063}$ & \textcolor{red}{$0.000_{\pm 0.000}$} \\
wdbc & $1.000_{\pm 0.000}$ & $1.000_{\pm 0.000}$ & $0.092_{\pm 0.111}$ & \textcolor{red}{$0.000_{\pm 0.000}$} & $1.000_{\pm 0.000}$ & $0.998_{\pm 0.001}$ & $1.000_{\pm 0.000}$ & - & $0.967_{\pm 0.021}$ & $0.261_{\pm 0.179}$ & \textcolor{red}{$0.000_{\pm 0.000}$} \\
\bottomrule
\end{tabular}

    }
\end{table*}
\begin{table*}[t]
    \centering
    \caption{Raw MLE performances as shown in Table~\ref{tab:ml-utility}, but experiments with privacy risks as indicated in Table~\ref{tab:dcr} highlighted and excluded from ranking. TTF-NM is not shown because it is not intended for a balanced objective. Baselines with more than half of the datasets having privacy concerns are removed.}
    \vspace{-0.3em}
    \label{tab:balance}
    \setlength{\tabcolsep}{1.1pt}
    \resizebox{\linewidth}{!}{

\begin{tabular}{ll>{\columncolor{gray!20}}ccccccc>{\columncolor{red!10}}c>{\columncolor{red!10}}c}
\toprule
Dataset & ML & real & CTAB+& TTVAE & TabSyn & GReaT & RTF & TabMT & TTF-S & TTF-L \\
\midrule
\multirow{3}{*}{adult} & LN & 0.914 & $0.904_{\pm 0.001}$ & $0.875_{\pm 0.010}$ & $0.910_{\pm 0.002}$ & $\boldsymbol{0.911}_{\pm 0.000}$ & \textcolor{red}{$0.911_{\pm 0.000}$} & $0.911_{\pm 0.000}$ & $0.906_{\pm 0.001}$ & $\boldsymbol{\underline{0.911}}_{\pm 0.001}$ \\
 & RF & 0.910 & $0.902_{\pm 0.002}$ & $0.869_{\pm 0.011}$ & $0.905_{\pm 0.003}$ & $0.905_{\pm 0.000}$ & \textcolor{red}{$0.908_{\pm 0.000}$} & $\underline{\boldsymbol{0.907}}_{\pm 0.002}$ & $0.901_{\pm 0.001}$ & $\boldsymbol{0.905}_{\pm 0.000}$ \\
 & XGB & 0.914 & $0.905_{\pm 0.001}$ & $0.870_{\pm 0.012}$ & $0.910_{\pm 0.003}$ & $0.911_{\pm 0.000}$ & \textcolor{red}{$0.912_{\pm 0.001}$} & $\boldsymbol{0.911}_{\pm 0.001}$ & $0.906_{\pm 0.001}$ & $\boldsymbol{\underline{0.911}}_{\pm 0.001}$ \\
\midrule
\multirow{3}{*}{bank} & LN & 0.907 & $0.901_{\pm 0.003}$ & $0.885_{\pm 0.004}$ & $0.900_{\pm 0.005}$ & $\underline{\boldsymbol{0.907}}_{\pm 0.000}$ & $0.904_{\pm 0.000}$ & $0.862_{\pm 0.005}$ & $0.906_{\pm 0.000}$ & $\boldsymbol{0.906}_{\pm 0.000}$ \\
 & RF & 0.930 & $0.902_{\pm 0.002}$ & $0.884_{\pm 0.009}$ & $0.907_{\pm 0.006}$ & $0.907_{\pm 0.001}$ & $\boldsymbol{0.913}_{\pm 0.001}$ & $0.886_{\pm 0.011}$ & $0.913_{\pm 0.002}$ & $\underline{\boldsymbol{0.918}}_{\pm 0.001}$ \\
 & XGB & 0.936 & $0.906_{\pm 0.000}$ & $0.877_{\pm 0.011}$ & $0.913_{\pm 0.007}$ & $0.910_{\pm 0.001}$ & $0.918_{\pm 0.002}$ & $0.892_{\pm 0.007}$ & $\boldsymbol{0.919}_{\pm 0.002}$ & $\underline{\boldsymbol{0.925}}_{\pm 0.003}$ \\
\midrule
\multirow{3}{*}{boston} & LN & 0.590 & $0.264_{\pm 0.034}$ & $0.000_{\pm 0.079}$ & $\boldsymbol{0.533}_{\pm 0.036}$ & $0.397_{\pm 0.000}$ & $0.529_{\pm 0.028}$ & $0.409_{\pm 0.015}$ & $0.525_{\pm 0.035}$ & $\underline{\boldsymbol{0.575}}_{\pm 0.011}$ \\
 & RF & 0.662 & $0.129_{\pm 0.015}$ & $0.124_{\pm 0.068}$ & $\boldsymbol{0.658}_{\pm 0.025}$ & $0.281_{\pm 0.001}$ & $0.584_{\pm 0.006}$ & $0.439_{\pm 0.096}$ & $0.597_{\pm 0.038}$ & $\underline{\boldsymbol{0.694}}_{\pm 0.041}$ \\
 & XGB & 0.701 & $0.097_{\pm 0.078}$ & $0.211_{\pm 0.121}$ & $\boldsymbol{0.647}_{\pm 0.017}$ & $0.277_{\pm 0.014}$ & $0.578_{\pm 0.056}$ & $0.446_{\pm 0.115}$ & $0.594_{\pm 0.061}$ & $\underline{\boldsymbol{0.665}}_{\pm 0.026}$ \\
\midrule
\multirow{3}{*}{breast} & LN & 0.985 & $0.913_{\pm 0.091}$ & $0.963_{\pm 0.006}$ & $\underline{\boldsymbol{0.987}}_{\pm 0.001}$ & $0.985_{\pm 0.000}$ & \textcolor{red}{$0.987_{\pm 0.002}$} & $0.983_{\pm 0.004}$ & $0.981_{\pm 0.005}$ & $\boldsymbol{0.987}_{\pm 0.002}$ \\
 & RF & 0.985 & $0.968_{\pm 0.012}$ & $0.979_{\pm 0.004}$ & $\boldsymbol{0.981}_{\pm 0.002}$ & $0.979_{\pm 0.004}$ & \textcolor{red}{$0.980_{\pm 0.000}$} & $0.978_{\pm 0.002}$ & $0.978_{\pm 0.001}$ & $\underline{\boldsymbol{0.983}}_{\pm 0.002}$ \\
 & XGB & 0.984 & $0.958_{\pm 0.025}$ & $0.970_{\pm 0.009}$ & $\underline{\boldsymbol{0.987}}_{\pm 0.001}$ & $0.982_{\pm 0.002}$ & \textcolor{red}{$0.982_{\pm 0.002}$} & $0.982_{\pm 0.002}$ & $0.982_{\pm 0.002}$ & $\boldsymbol{0.983}_{\pm 0.003}$ \\
\midrule
\multirow{3}{*}{credit} & LN & 0.836 & $0.538_{\pm 0.093}$ & $0.500_{\pm 0.000}$ & $\boldsymbol{0.780}_{\pm 0.017}$ & $0.665_{\pm 0.000}$ & \textcolor{red}{$0.801_{\pm 0.006}$} & $0.773_{\pm 0.038}$ & $\underline{\boldsymbol{0.816}}_{\pm 0.020}$ & $0.755_{\pm 0.011}$ \\
 & RF & 0.837 & $0.505_{\pm 0.097}$ & $0.500_{\pm 0.000}$ & $\underline{\boldsymbol{0.791}}_{\pm 0.007}$ & $0.722_{\pm 0.014}$ & \textcolor{red}{$0.820_{\pm 0.010}$} & $0.763_{\pm 0.019}$ & $\boldsymbol{0.781}_{\pm 0.026}$ & $0.741_{\pm 0.027}$ \\
 & XGB & 0.844 & $0.509_{\pm 0.088}$ & $0.500_{\pm 0.000}$ & $\boldsymbol{0.795}_{\pm 0.014}$ & $0.739_{\pm 0.006}$ & \textcolor{red}{$0.797_{\pm 0.039}$} & $0.761_{\pm 0.030}$ & $\underline{\boldsymbol{0.796}}_{\pm 0.011}$ & $0.764_{\pm 0.014}$ \\
\midrule
\multirow{3}{*}{diabetes} & LN & 0.884 & $0.817_{\pm 0.028}$ & $0.880_{\pm 0.003}$ & \textcolor{red}{$0.883_{\pm 0.006}$} & $0.858_{\pm 0.000}$ & $0.878_{\pm 0.003}$ & $0.855_{\pm 0.014}$ & $\underline{\boldsymbol{0.890}}_{\pm 0.010}$ & $\boldsymbol{0.886}_{\pm 0.001}$ \\
 & RF & 0.869 & $0.802_{\pm 0.014}$ & $0.849_{\pm 0.004}$ & \textcolor{red}{$0.844_{\pm 0.012}$} & $0.824_{\pm 0.012}$ & $0.841_{\pm 0.010}$ & $0.845_{\pm 0.009}$ & $\boldsymbol{0.855}_{\pm 0.014}$ & $\underline{\boldsymbol{0.859}}_{\pm 0.010}$ \\
 & XGB & 0.867 & $0.811_{\pm 0.005}$ & $\boldsymbol{0.851}_{\pm 0.001}$ & \textcolor{red}{$0.847_{\pm 0.009}$} & $0.831_{\pm 0.003}$ & $0.849_{\pm 0.010}$ & $0.840_{\pm 0.009}$ & $\underline{\boldsymbol{0.852}}_{\pm 0.015}$ & $0.848_{\pm 0.017}$ \\
\midrule
\multirow{3}{*}{iris} & LN & 1.000 & $0.272_{\pm 0.068}$ & $0.996_{\pm 0.006}$ & $\boldsymbol{0.999}_{\pm 0.002}$ & $0.985_{\pm 0.000}$ & $0.983_{\pm 0.015}$ & $\underline{\boldsymbol{1.000}}_{\pm 0.000}$ & $0.971_{\pm 0.018}$ & $0.997_{\pm 0.005}$ \\
 & RF & 1.000 & $0.396_{\pm 0.099}$ & $0.992_{\pm 0.012}$ & $\underline{\boldsymbol{1.000}}_{\pm 0.000}$ & $\underline{\boldsymbol{1.000}}_{\pm 0.000}$ & $0.993_{\pm 0.005}$ & $0.993_{\pm 0.005}$ & $0.999_{\pm 0.002}$ & $0.988_{\pm 0.017}$ \\
 & XGB & 1.000 & $0.211_{\pm 0.042}$ & $0.999_{\pm 0.002}$ & $\underline{\boldsymbol{1.000}}_{\pm 0.000}$ & $\underline{\boldsymbol{1.000}}_{\pm 0.000}$ & $0.998_{\pm 0.003}$ & $\underline{\boldsymbol{1.000}}_{\pm 0.000}$ & $\underline{\boldsymbol{1.000}}_{\pm 0.000}$ & $\underline{\boldsymbol{1.000}}_{\pm 0.000}$ \\
\midrule
\multirow{3}{*}{qsar} & LN & 0.906 & $0.716_{\pm 0.085}$ & $0.669_{\pm 0.120}$ & $0.867_{\pm 0.012}$ & $0.673_{\pm 0.000}$ & $\boldsymbol{0.880}_{\pm 0.003}$ & $\underline{\boldsymbol{0.887}}_{\pm 0.005}$ & $0.874_{\pm 0.013}$ & $0.876_{\pm 0.006}$ \\
 & RF & 0.936 & $0.648_{\pm 0.036}$ & $0.666_{\pm 0.117}$ & $0.882_{\pm 0.006}$ & $0.616_{\pm 0.009}$ & $\underline{\boldsymbol{0.907}}_{\pm 0.007}$ & $0.880_{\pm 0.008}$ & $0.873_{\pm 0.006}$ & $\boldsymbol{0.884}_{\pm 0.004}$ \\
 & XGB & 0.921 & $0.731_{\pm 0.020}$ & $0.658_{\pm 0.114}$ & $0.860_{\pm 0.016}$ & $0.602_{\pm 0.020}$ & $\underline{\boldsymbol{0.897}}_{\pm 0.010}$ & $0.873_{\pm 0.009}$ & $0.862_{\pm 0.005}$ & $\boldsymbol{0.884}_{\pm 0.004}$ \\
\midrule
\multirow{3}{*}{wdbc} & LN & 0.993 & $0.929_{\pm 0.053}$ & $0.976_{\pm 0.019}$ & $\underline{\boldsymbol{0.992}}_{\pm 0.001}$ & $0.979_{\pm 0.000}$ & $\boldsymbol{0.988}_{\pm 0.003}$ & $0.984_{\pm 0.006}$ & $0.980_{\pm 0.007}$ & $0.979_{\pm 0.007}$ \\
 & RF & 0.976 & $0.919_{\pm 0.024}$ & $0.952_{\pm 0.041}$ & $\underline{\boldsymbol{0.984}}_{\pm 0.004}$ & $0.976_{\pm 0.003}$ & $0.980_{\pm 0.012}$ & $0.978_{\pm 0.003}$ & $\boldsymbol{0.981}_{\pm 0.001}$ & $0.975_{\pm 0.002}$ \\
 & XGB & 0.990 & $0.930_{\pm 0.021}$ & $0.949_{\pm 0.035}$ & $0.986_{\pm 0.002}$ & $0.973_{\pm 0.003}$ & $\underline{\boldsymbol{0.988}}_{\pm 0.002}$ & $0.982_{\pm 0.004}$ & $\boldsymbol{0.987}_{\pm 0.003}$ & $0.985_{\pm 0.006}$ \\
\bottomrule
\end{tabular}

    }
\end{table*}

Fidelity metrics are calculated using SDMetrics~\citep{sdmetrics}'s public SDK.

\subsection{Distance to Closest Record (DCR) Metrics Implementation}
\label{app:setup:dcr-metric}

The way the DCR values are reported differ from paper to paper. Also, existing papers often report a value and claim some specific value to be an ideal value, overlooking the fact that better privacy usually sacrifices quality. Thus, in this paper, we instead apply a statistical test on DCR values to \textit{validate} privacy-preserving capability of the models, instead of \textit{evaluating} it, which is likely also more useful in practical privacy-preserving data sharing cases. 

We compare the DCR to the real training dataset $\mathbf{X}$ of a hold-out real dataset $\widehat{\mathbf{X}}$ with a synthetic dataset of the same size $\mathbf{X}'$, and privacy-perserving means that DCRs calculated on the latter is not smaller than DCRs calculated on the former. We test this using Mann-Whitney U Test~\citep{mwu}, with the null hypothesis $H_0$ being that the distance between $\widehat{\mathbf{X}}$ and $\mathbf{X}$ is greater than or equal to $\mathbf{X}'$ and $\mathbf{X}$, and if the $p$-value is less than 0.05, $\mathbf{X}'$ is closer to $\mathbf{X}$ than $\widehat{\mathbf{X}}$, implying a risk of privacy leakage. 

The data is preprocessed by quantile transformation (at 1000 quantiles, which is the same number of bins for the quantile quantizer in \framework) for numeric values and one-hot encoding for categorical values for distance calculation. Cosine distance (calculated using \verb|sklearn|) is used to evaluate the distance between records. As a reference, we calculate the cosine distances between the real test set and the real train set, obtaining the minimum distance per record for the test set. Similarly, we compute the cosine distances between synthetic data (with the same number of rows as the real test set) and the real train set, and also obtain the minimum distances. These distances are compared with the reference values to assess the similarity and privacy-preserving characteristics of the synthetic data.

\section{Raw Experimental Results}
\label{app:exp}

\subsection{Raw Synthetic Data Quality Results}
\label{app:exp:utility}
The raw experiment results for MLE is shown in Table~\ref{tab:ml-utility}. Out of the 27 MLE scores, \framework-NM is the best in 11, while the best baseline achieves the best in at most 4).

The raw experiment results for fidelity is shown in Table~\ref{tab:fidelity}.

\subsection{Raw Privacy Results}
\label{app:exp:privacy}

The raw experiment results for DCR is shown in Table~\ref{tab:dcr}.

\paragraph{Notes on the \textbf{N}o \textbf{M}ask (NM) Setting and Quality-privacy Dilemma.}
\label{app:exp:npsetting}
Summarizing the settings of the non-private version mentioned in Appendix~\ref{app:model}, TTF-NM in the paper means TTF-L model size with 0 masking rate and technically no early stopping.
Such a setting makes the model prone to memorizing exact values in the training data, which is verified by our experiment results on DCR.
Nevertheless, in a non-private setting, we regard this as permissible because privacy is not a concern.

In fact, for any model, synthetic data quality and its privacy are usually a dilemma, which is also consistent with our experiment results on different baselines.
For example, the baseline model with the best MLE result overall is TabuLa~\citep{tabula}, but it has a privacy issue in \textit{all} datasets. 
In particular, we found that TabuLa has a DCR of 0 on more than 50\% of the rows in many datasets, and a DCR of smaller than $1\times10^{-6}$ on more than 80\% of the rows. This indicates a severe memorization issue.
In comparison, the worst-performing baseline model we tested, CTAB-GAN+~\citep{ctabganp}, does not have any issue with privacy according to the DCR scores.

\paragraph{Balance of Quality and Utility.}
Given the dilemma between quality and utility, the general objective of a tabular data generator should be a good balance between quality and privacy, instead of optimizing both simultaneously.
Table~\ref{tab:balance} shows the MLE scores with privacy warnings. \framework demonstrates a good balance.

\subsection{Computation Time}
\label{app:exp:time}

\begin{table*}[t]
    \caption{Comparison between different models in terms of time taken for training and generation. The values are in unit of seconds.}
    \vspace{-0.3em}
    \label{tab:time}
    \setlength{\tabcolsep}{1pt}
    \resizebox{\linewidth}{!}{
    \begin{tabular}{llcccccccc>{\columncolor{red!10}}c>{\columncolor{red!10}}c>{\columncolor{red!10}}c>{\columncolor{red!10}}c>{\columncolor{red!10}}c}
\toprule
Dataset & Metric & CTAB+ & TTVAE & TabSyn & FD & GReaT & RTF & TabMT & TabuLa & TTF-S & TTF-L \\
\midrule
\multirow{2}{*}{adult} & Train & $787.379_{\pm 3.334}$ & $194.718_{\pm 0.085}$ & $1199.823_{\pm 146.744}$ & $3857.698_{\pm 10.002}$ & $10397.040_{\pm 9.203}$ & $1875.904_{\pm 34.663}$ & $6029.515_{\pm 3.139}$ & $2859.160_{\pm 6.858}$ & $1096.467_{\pm 143.847}$ & $2503.830_{\pm 564.922}$ \\
 & Generate & $0.429_{\pm 0.001}$ & $8.546_{\pm 0.029}$ & $2.409_{\pm 0.010}$ & $48.730_{\pm 13.758}$ & $138.023_{\pm 0.205}$ & $89.466_{\pm 4.301}$ & $416.206_{\pm 1.612}$ & $249.319_{\pm 0.485}$ & $41.482_{\pm 8.907}$ & $107.193_{\pm 24.859}$ \\
\midrule
\multirow{2}{*}{bank} & Train & $737.314_{\pm 0.138}$ & $192.617_{\pm 1.959}$ & $1865.473_{\pm 418.665}$ & $4172.922_{\pm 3.229}$ & $10295.769_{\pm 21.720}$ & $835.167_{\pm 1.619}$ & $7032.875_{\pm 4.217}$ & $3079.641_{\pm 2.237}$ & $1115.787_{\pm 108.995}$ & $2551.349_{\pm 724.774}$ \\
 & Generate & $1.126_{\pm 0.006}$ & $7.341_{\pm 0.015}$ & $3.009_{\pm 0.063}$ & $31.609_{\pm 0.373}$ & $149.872_{\pm 0.104}$ & $106.036_{\pm 0.080}$ & $333.882_{\pm 118.034}$ & $234.806_{\pm 0.044}$ & $24.972_{\pm 3.040}$ & $58.530_{\pm 16.181}$ \\
\midrule
\multirow{2}{*}{boston} & Train & $8.396_{\pm 0.048}$ & $4.664_{\pm 0.129}$ & $862.625_{\pm 48.241}$ & $47.040_{\pm 1.544}$ & $121.619_{\pm 0.090}$ & $49.693_{\pm 3.175}$ & $225.624_{\pm 0.310}$ & $47.811_{\pm 0.079}$ & $67.370_{\pm 43.192}$ & $229.997_{\pm 98.326}$ \\
 & Generate & $0.083_{\pm 0.001}$ & $0.172_{\pm 0.029}$ & $0.156_{\pm 0.003}$ & $0.447_{\pm 0.007}$ & $2.422_{\pm 0.003}$ & $3.059_{\pm 0.003}$ & $2.524_{\pm 0.001}$ & $3.012_{\pm 0.013}$ & $0.533_{\pm 0.028}$ & $1.566_{\pm 0.401}$ \\
\midrule
\multirow{2}{*}{breast} & Train & $8.479_{\pm 0.026}$ & $6.355_{\pm 0.592}$ & $1037.126_{\pm 281.166}$ & $32.364_{\pm 0.506}$ & $153.352_{\pm 0.134}$ & $25.537_{\pm 1.419}$ & $214.935_{\pm 0.461}$ & $33.100_{\pm 0.059}$ & $122.137_{\pm 6.839}$ & $344.247_{\pm 59.644}$ \\
 & Generate & $0.056_{\pm 0.008}$ & $0.144_{\pm 0.032}$ & $0.249_{\pm 0.003}$ & $0.771_{\pm 0.005}$ & $3.434_{\pm 0.016}$ & $0.759_{\pm 0.003}$ & $1.653_{\pm 0.013}$ & $6.652_{\pm 0.011}$ & $0.620_{\pm 0.033}$ & $1.463_{\pm 0.208}$ \\
\midrule
\multirow{2}{*}{credit} & Train & $8.400_{\pm 0.035}$ & $12.288_{\pm 2.809}$ & $726.006_{\pm 67.662}$ & $185.868_{\pm 5.293}$ & $277.555_{\pm 0.388}$ & $63.648_{\pm 0.758}$ & $703.238_{\pm 1.755}$ & $78.538_{\pm 0.338}$ & $179.304_{\pm 56.926}$ & $458.389_{\pm 89.990}$ \\
 & Generate & $0.043_{\pm 0.000}$ & $0.211_{\pm 0.115}$ & $0.167_{\pm 0.003}$ & $1.127_{\pm 0.010}$ & $5.845_{\pm 0.008}$ & $2.407_{\pm 0.171}$ & $5.434_{\pm 0.013}$ & $4.868_{\pm 0.014}$ & $0.604_{\pm 0.123}$ & $1.558_{\pm 0.227}$ \\
\midrule
\multirow{2}{*}{diabetes} & Train & $8.524_{\pm 0.407}$ & $6.673_{\pm 0.480}$ & $758.140_{\pm 41.333}$ & $30.324_{\pm 2.067}$ & $130.793_{\pm 0.077}$ & $37.372_{\pm 3.118}$ & $220.147_{\pm 0.667}$ & $40.229_{\pm 0.050}$ & $70.158_{\pm 26.658}$ & $242.176_{\pm 48.014}$ \\
 & Generate & $0.047_{\pm 0.001}$ & $0.064_{\pm 0.001}$ & $0.139_{\pm 0.000}$ & $0.383_{\pm 0.003}$ & $1.496_{\pm 0.008}$ & $1.504_{\pm 0.007}$ & $0.786_{\pm 0.000}$ & $3.994_{\pm 0.007}$ & $0.343_{\pm 0.062}$ & $1.169_{\pm 0.364}$ \\
\midrule
\multirow{2}{*}{iris} & Train & $7.267_{\pm 0.021}$ & $2.489_{\pm 0.005}$ & $620.864_{\pm 35.603}$ & $5.055_{\pm 1.403}$ & $24.975_{\pm 0.018}$ & $15.369_{\pm 1.242}$ & $32.846_{\pm 0.011}$ & \multirow{2}{*}{-} & $86.143_{\pm 41.970}$ & $126.315_{\pm 29.440}$ \\
 & Generate & $0.033_{\pm 0.006}$ & $0.025_{\pm 0.005}$ & $0.098_{\pm 0.002}$ & $0.095_{\pm 0.002}$ & $0.360_{\pm 0.015}$ & $0.225_{\pm 0.001}$ & $0.122_{\pm 0.031}$ &  & $0.209_{\pm 0.065}$ & $0.483_{\pm 0.078}$ \\
\midrule
\multirow{2}{*}{qsar} & Train & $18.389_{\pm 0.217}$ & $13.144_{\pm 1.576}$ & $694.890_{\pm 76.489}$ & $664.578_{\pm 3.603}$ & $566.658_{\pm 23.451}$ & $152.950_{\pm 0.819}$ & $1723.824_{\pm 2.596}$ & $188.475_{\pm 0.374}$ & $81.433_{\pm 14.458}$ & $334.639_{\pm 74.700}$ \\
 & Generate & $0.206_{\pm 0.006}$ & $0.432_{\pm 0.006}$ & $0.336_{\pm 0.005}$ & $2.608_{\pm 0.007}$ & $46.712_{\pm 3.062}$ & $10.997_{\pm 0.072}$ & $23.723_{\pm 0.025}$ & $20.261_{\pm 5.972}$ & $1.270_{\pm 0.094}$ & $4.118_{\pm 0.606}$ \\
\midrule
\multirow{2}{*}{wdbc} & Train & $10.795_{\pm 0.089}$ & $8.168_{\pm 0.171}$ & $631.546_{\pm 68.070}$ & $306.120_{\pm 3.189}$ & $265.340_{\pm 2.411}$ & $183.908_{\pm 1.047}$ & $643.120_{\pm 1.028}$ & \multirow{2}{*}{-} & $31.884_{\pm 4.836}$ & $169.320_{\pm 54.606}$ \\
 & Generate & $0.159_{\pm 0.001}$ & $0.254_{\pm 0.001}$ & $0.276_{\pm 0.005}$ & $1.267_{\pm 0.023}$ & $65.821_{\pm 1.106}$ & $9.936_{\pm 0.017}$ & $13.893_{\pm 0.068}$ &  & $0.585_{\pm 0.010}$ & $1.381_{\pm 0.278}$ \\
\bottomrule
\end{tabular}

    }
\end{table*}

The raw computation time of training and generation is shown in Table~\ref{tab:time}. 
\framework trains two separate models for cross-validation, and training stops mainly by the validation loss, and occasionally by a maximum number of steps, so the relative training time compared to other auto-regressive transformer baselines is higher, but when trained on large datasets, \framework shows its efficiency advantage more obviously.

\section{Limitations and Future Work}
\label{app:limitation}

\paragraph{Limited capability by transformer backbones.} 
Although this paper demonstrates a strong capability of \framework and a potential of comparable performance with a much smaller version, it still inherits the problems of the backbone transformer. In particular, \framework's tabular data representation effectively reduces the number of tokens to one to two times the number of columns, significantly smaller than the 5-10 times of some baseline auto-regressive transformers, but is still linear to the number of columns.
If the chosen transformer backbone has limited performance or impractical memory requirement for the sequence length induced by a dataset with a tremendous number of columns, the performance of \framework is subject to the capability of the transformer backbone. Fortunately, most transformers are believed to function where with a few thousand of tokens in a sequence, endorsing the encoding of tabular datasets of the number of columns up to thousands, which covers most practical use cases.
Moreover, the design of \framework makes the transformer backbone flexibly configurable, and any modern or future more powerful versions of transformers can be substituted in.

\paragraph{Optimizing tree-based model.} 
\frameworkname\space uses LightGBM fitted on the target column to introduce tabular inductive biases from tree-based models, While this is effective and easy to implement, modified tree-based models catered for guiding the generation may further improve the performance, such as using trees from a generative tree-based model (e.g. ARF), and selecting a better target column using some heuristics.

\paragraph{Smoothing quantized values.} 
The dual-quantization tokenization may be further optimized with some sampling around the quantile values if use case requires smoother values in continuous values. Adapted ordinal embeddings with continuous raw values as additional input can also be applied.
In this paper, we want to focus on showing the effectiveness of quantization and methods to handle quantized values. We leave the integration of raw continuous values on top of these methods for future works, as performance without raw continuous values is readily outstanding.
Moreover, using quantized values only makes it much easier to change the backbone transformer.

\paragraph{Masked vs. Auto-regressive transformers.} 
TabMT has demonstrated the effectiveness of a masked transformer with iterative decoding for tabular data generation, but TabMT does not provide ablation study result on masked vs. auto-regressive transformers. Thus, this approach may be used in place of a causal language model to further improve synthetic data quality. Further exploration could be interesting.
In this paper, we want to focus on the most intuitive and simple usage of generative transformers, namely, auto-regressive generation, as the performance of this simple case is readily outstanding.
Moreover, using the standard auto-regressive setting makes it much easier to integrate with other training optimizations available on open-source seq2seq trainers (e.g., Hugging Face).

\section{Broader Impact}
\vspace{-0.5em}
\label{app:impact}
\paragraph{Extension to general tabular tasks.} 
The idea of introducing tree-based models to transformers on tabular data may not be limited to generation tasks, but also applies to tasks like classification and regression. 
We envision tabular models designed for other tasks inspired by \framework's design of the integration of a tree-based model to transformer would improve its performance by taking the advantage of both.

\paragraph{General ordinal token space learning.}
The proposed methods to learn the ordinal token space, including the embedding and loss, can be applied to any other task involving ordinal tokens, especially when the ordinal tokens' relation is absolute and monotonic instead of relative and non-periodic.

\end{document}